\documentclass{article}

\usepackage{arxiv}

\usepackage[utf8]{inputenc} 
\usepackage[T1]{fontenc}    
\usepackage{hyperref}       
\usepackage{url}            
\usepackage{booktabs}       
\usepackage{amsfonts}       
\usepackage{nicefrac}       
\usepackage{microtype}      
\usepackage{lipsum}		
\usepackage{graphicx}
\usepackage{natbib}
\usepackage{doi}

\setcitestyle{authoryear,round,citesep={;},aysep={,},yysep={;}}
\usepackage{listings}
\usepackage{amsmath}
\usepackage{amsthm}
\usepackage{thm-restate}
\usepackage{xcolor}
\usepackage{verbatim}
\usepackage{comment}
\usepackage{enumitem} 
\usepackage{booktabs}
\usepackage{multirow}
\usepackage{tabularx}
\usepackage{enumitem}
\usepackage{bookmark}

\newcommand{\R}{\mathbb{R}}
\newcommand{\evalat}[2]{\left.#1\right|_{#2}}
\newenvironment{psmallmatrix}
  {\left(\begin{smallmatrix}}
  {\end{smallmatrix}\right)}

\newtheorem{theorem}{Theorem}
\newtheorem{proposition}[theorem]{Proposition}
\newtheorem{lemma}[theorem]{Lemma}

\newtheorem{definition}[theorem]{Definition}

\title{On Parameter Symmetries and Conservation Laws in Gradient Flow}

\date{} 					

\author{Khang Nguyen
		\\
	UCLA\\
	\texttt{khang@math.ucla.edu} \\
	\And
	Guido Mont\'ufar\\
	UCLA \& MPI MiS\\
	\texttt{montufar@math.ucla.edu} \\
}

\renewcommand{\headeright}{}
\renewcommand{\undertitle}{}

\hypersetup{
pdftitle={On Parameter Symmetries and Conservation Laws in Gradient Flow},
pdfsubject={optimization, machine learning, deep learning},
pdfauthor={Khang Nguyen, Guido Mont\'ufar},
pdfkeywords={symmetry, parameter symmetry, conservation law, gradient flow, identifiability, differential geometry, neural network},
}

\begin{document}
\maketitle
\vspace{-0.1in}

\begin{abstract}
	Parameter space symmetries and conservation laws play an important role in understanding the loss landscapes and implicit biases of neural networks. 
	Inspired by Noether's theorem in physics, prior works have sought to derive conservation laws under gradient flow from parameter symmetries, but the scope and limitations of this connection remain unclear. 
	We develop a unified geometric framework that clarifies the precise relationship between the two notions, 
	including the conditions under which symmetries correspond to conservation laws. 
	We introduce a notion of compositional identifiability and use it to establish a general inheritance principle for complete characterizations of symmetries and conservation laws in multilayer networks. 
	We apply the framework to multi-head and grouped-query attention, polynomial neural networks, and square deep linear networks. 
\end{abstract}

\section{Introduction}

Modern deep learning models give rise to highly complex, non-convex training objectives.
Optimizing such models requires careful consideration of both the model architecture and the optimization procedure to obtain successful training \citep{glorot10understanding}, fast and stable convergence \citep{santurkar2018how}, and to induce favourable implicit biases for generalization \citep{wang21theimplicitbias}. 
Understanding these optimization characteristics is thus a fundamental problem in deep learning, and has been the subject of extensive and sustained research \citep[see, e.g.,][]{sun2019optimizationdeeplearningtheory, vardi2023implicitbias}. 

Parameter space symmetries offer one avenue to study these problems, as they characterize the redundancy in the loss landscape 
\citep{pmlr-v139-simsek21a,zhao_symmetry_2025}. 
Alternatively, conservation laws constrain the optimization trajectory from initialization to convergence, providing an explicit characterization of the resulting optimization bias. 
Inspired by Noether's theorem \citep{Noether1918}, which roughly states that every continuous symmetry of a physical system induces a conservation law, prior works have developed theoretical frameworks to derive conservation laws from symmetries in neural network training dynamics \citep{kunin2021neuralmechanicssymmetrybroken,gluch2021noetherthingschangestay}. 

However, for gradient flow, the relationship is more nuanced than these 
frameworks suggest. 
Recent work raises the question of whether all conserved quantities in neural networks arise from known symmetries \citep{zhao_symmetry_2025}. 
Other works point at a more fundamental connection between symmetries and conservation laws through their underlying vector fields 
\citep{marcotte2025transformativeconservativeconservationlaws,gegenfurtner2026symmetries}. 
Based on these observations, we develop a geometric framework 
that precisely characterizes the relationship between symmetries and conservation laws in gradient flow.

\subsection{Main contributions} 
\begin{itemize}
    \item 
    We develop a unified vector-field framework for loss symmetries and conservation laws under gradient flow and formalize the gap between them, 
    showing that independent conservation laws induce independent loss symmetries, but not vice versa. 

    \item We connect this framework to functional symmetries and model identifiability, and derive general results on the inheritance and completeness of symmetries and conservation laws in compositional models which can be applied to multilayer networks. 

    \item We apply our framework to establish the complete set of conservation laws for multi-head self-attention, grouped-query attention and deep polynomial neural networks, and provide a bound for the number of independent conservation laws in square deep linear networks. 
\end{itemize}

\subsection{Related works}

\textbf{Parameter symmetry and identifiability.} 
Neural networks often admit transformations of the parameters that leave the represented function unchanged \citep{SUSSMANN1992589,NIPS1993_e49b8b40}. 
Some of these parameter symmetries are well known, while others depend more subtly on the architecture and activation function. 
They arise, e.g., in 
linear networks \citep{baldi1989,trager2020Pure}, 
polynomial neural networks \citep{kileel2019polynomial,usevich2025identifiability}, 
ReLU networks \citep{pmlr-v202-grigsby23a,gegenfurtner2026symmetries}, 
networks with batch normalization \citep{kunin2021neuralmechanicssymmetrybroken}, 
 multi-head self-attention \citep{tran2025equivariantneuralfunctionalnetworks}, 
 and 
 symmetries induced by particular activation functions \citep{VLACIC2021107485,zhao_symmetries_2023}. 
 This has led to new theoretical understandings of optimization landscapes \citep{ziyin2024symmetry,tran2025on} and practical techniques for improving the convergence and generalization of neural networks \citep{zhao2024improving}.
 For a survey of \mbox{parameter symmetries in neural networks and their applications, we refer to \cite{zhao_symmetry_2025}.}

\textbf{Conservation laws in gradient flow.} 
Conservation laws first appeared as 
balancedness conditions in deep linear networks \citep{saxe2014exact} and have been useful in understanding the convergence behavior \citep{arora2018a, du2018algorithmicregularization}, acceleration characteristics \citep{arora18aon}, and implicit bias of training dynamics 
\citep{arora2019implicitregularization}. 
\citet{le2022traininginvarianceslowrankphenomenon} derived conservation laws for
nonlinear layers,
including ResNet and convolutional layers, and used them to explain the low-rank phenomenon in the training of deep feedforward networks. 
Recent work by \citet{marcotte2024abidelawfollowflow} 
network architecture, with applications to shallow linear and ReLU networks. 
Subsequent work by \cite{marcotte2025transformativeconservativeconservationlaws,tran2025equivariantneuralfunctionalnetworks} 
derived the complete set of conservation laws for 
further layer types, including single and multi-head self-attention, ResNet, modern activations,
and Mixture-of-Experts. 
These works also highlight the difficulty of
extending such results to more complex architectures, such as multilayer networks. 
Other active topics include conservation laws in other training dynamics \citep{marcotte_momentum_2024}, data-dependent conservation laws \citep{galley2026conservationlawsdatasymmetry}, and topology and geometry of the learning space \citep{nurisso2024topological, nurisso2026topology}.

\textbf{Connection between symmetries and conservation laws.} 
Noether's theorem establishes that every continuous symmetry in a Lagrangian system corresponds to a conservation law \citep{Noether1918}.
Prior works by \cite{kunin2021neuralmechanicssymmetrybroken,tanaka2021noetherslearningdynamics, gluch2021noetherthingschangestay,zhao_symmetries_2023} have sought to establish a similar framework in neural network training dynamics, but the constructions have remained limited to specific types of symmetries and do not give a complete characterization of the relationship between symmetries and conservation laws for gradient flows. 
Several recent works have observed that symmetries of the loss function and conservation laws of the gradient flow are related through their underlying vector fields \citep{josz2025subdifferentiationsymmetry,white2025beyondlagrangians,marcotte2025transformativeconservativeconservationlaws,gegenfurtner2026symmetries}.
Our framework builds upon this geometric observation. We summarize and contrast prior works with ours in Appendix~\ref{sec:comparison}.

\section{Preliminaries} 
\label{sec:prelim}

Let \(\Theta \subset \R^D\) be the parameter space, \(\mathcal X\) be the input space, and \(\mathcal Z\) be the output space of a network
\[
G: \Theta \times \mathcal X \to \mathcal Z, \quad (\theta, x) \mapsto G(\theta, x).
\]
Let \(\ell\) be a loss function that compares the model output \(z := G(\theta, x)\) against the label \(y \in \mathcal Y\), 
\[
\ell : \mathcal Z \times \mathcal Y \to \mathbb R, \quad (z, y) \mapsto \ell(z,y).
\]
Given a pair \((x, y) \in \mathcal X \times \mathcal Y\), the single-sample loss is
\(\ell_{(x,y)}(\theta) := \ell(G(\theta, x), y)\). 
For a dataset \(\mathcal D = (x^{(i)}, y^{(i)})_{i=1}^n \in (\mathcal X \times \mathcal Y)^\ast\), training seeks to minimize the empirical loss
\[
L_{\mathcal D}(\theta) = \frac 1 n \sum_{i=1}^n \ell\left(G(\theta, x^{(i)}), y^{(i)}\right) = \frac 1 n \sum_{i=1}^n \ell_{(x^{(i)}, y^{(i)})}(\theta).  
\]

In this paper, we specifically consider gradient flow as the training mechanism: 
\begin{equation}
\label{eqn:gradient-flow}
\dot \theta(t) = - \nabla_\theta L_{\mathcal D}(\theta) = - \frac 1 n \sum_{i=1}^n \nabla_\theta \ell_{(x^{(i)}, y^{(i)})}(\theta)     , \quad \theta(0) = \theta_{0} . 
\end{equation}
For this ODE to be well-defined by the Picard-Lindel\"of theorem, we assume throughout that \(\Omega\) is open 
and, for all \((x,y) \in \mathcal X \times \mathcal Y\), 
\(\ell_{(x,y)}(\theta)\) is continuously differentiable 
with locally Lipschitz gradient \(\nabla_\theta \ell_{(x,y)}(\theta)\) 
on \(\Theta\). 
This assumption is satisfied by neural network architectures that use smooth activation functions, such as sigmoid, tanh, GELU, Swish, SwiGLU, etc., and that are trained with smooth loss functions, such as the mean squared error loss and the logistic loss.

Moreover, for common loss functions whose value is bounded from below, the following proposition shows that the gradient flow is well-defined globally, i.e., it exists and is unique for all time \(t \in [0, \infty)\).
This permits us to talk about gradient flow corresponding to any finite dataset \(\mathcal D\), any initialization \(\theta_0 \in \Theta\), and for all time \(t \in [0, \infty)\).
\begin{restatable}{proposition}{existenceGradientFlow}
    \label{prop:existence_gradient_flow}
    Suppose \(\Theta = \R^D\), the loss function \(\ell\) is bounded from below, and for any \((x,y) \in \mathcal X \times \mathcal Y\), \(\ell_{(x,y)}(\theta) \in C^1(\Theta, \R)\) with locally lipschitz gradient \(\nabla_\theta \ell_{(x,y)}(\theta)\). Then for every finite nonempty dataset \(\mathcal D \in (\mathcal X \times \mathcal Y)^\ast\) and every initialization \(\theta_0 \in \Theta\), the gradient flow (\ref{eqn:gradient-flow}) admits a unique solution \(\theta(t)\) defined for all \(t \in [0, \infty)\).
\end{restatable}

The proof of this result, along with all other proofs in the paper, is deferred to the appendix.

\subsection{Conservation laws}
As defined by \cite{marcotte2024abidelawfollowflow}, a function \(h: \Theta \to \R\) is called a conservation law if, for any solution \(\theta(t)\) of the gradient flow equation (\ref{eqn:gradient-flow}) with respect to any dataset \(\mathcal D\) and initialization \(\theta_0\), the value of \(h\) remains constant along the trajectory, i.e., \(h(\theta(t)) = h(\theta_0)\) for all \(t\).
Corollary 2.6 and Proposition 2.7 from the same paper provide the following equivalent condition, which says that the conservativeness of any function \(h\) is completely characterized by its infinitesimal orthogonal alignment with respect to the gradient of the loss function.
\begin{proposition}
    \label{prop:conservation-law-condition}
    A continuously differentiable function \(h: \Theta \to \mathbb R\) is a conservation law if and only if \(\nabla_\theta h(\theta) \perp W_\theta\) for all \(\theta \in \Theta\), where
    \begin{equation}
        \label{eqn:gradhperpgradL}
        W_\theta 
        := 
        \operatorname{span}_{(x,y) \in \mathcal X \times \mathcal Y} \left\{[D_\theta G(\theta, x)]^\top \nabla_z \ell(G(\theta, x), y)\right\}
        =\operatorname{span}_{(x,y) \in \mathcal X \times \mathcal Y} \left\{\nabla \ell_{(x,y)}(\theta)\right\}.
    \end{equation}
\end{proposition}

We call \(W\) the \emph{gradient flow distribution}, which captures the directions in which the parameters can evolve under gradient flow for all possible data pairs \((x,y)\), and hence also the directions in which the parameters can evolve under gradient flow for any dataset \(\mathcal D\).
The terminology \emph{distribution} is used in the sense of differential geometry, where a distribution is an assignment of a subspace of the tangent space at each point on a manifold. 
In this case, the manifold is the parameter space \(\Theta\), and the tangent space at each point \(\theta\) is simply \(\R^D\).

A fundamental question is to characterize all conservation laws of a given network architecture and loss function.
To this end, \cite{marcotte2024abidelawfollowflow} gives a definition of independent conservation laws based on the concept of functional independence \citep{newns1967functionaldependence}.
We call a set of conservation laws complete if they are functionally independent and there are no other functionally independent conservation laws. 
We recall the definition of functional independence: 
\begin{definition}
    \label{def:independent_laws}
    A family of \(N\) functions \(\{h_1, \ldots, h_N\}\) in \(C^1(\Theta, \R)\) is said to be functionally independent if the vectors \(\{\nabla_\theta h_1(\theta), \ldots, \nabla_\theta h_N(\theta)\}\) are linearly independent for all \(\theta \in \Theta\).
\end{definition}

The definition is justified by Proposition 2.12 of \cite{marcotte2025transformativeconservativeconservationlaws}, which states that, given a complete set of conservation laws, any other conservation law can locally be expressed as a function of the existing conservation laws. 
In verifying this result, however, we found that the proof is insufficient for the stated conclusion. 
Specifically, the proof invokes 
Theorem~1 of \citet{newns1967functionaldependence}, which does not by itself show that
an additional dependent function can be expressed as a function of the prior independent functions. 
To address this gap, we provide a general local factorization result that 
captures the intended conclusion. 

\begin{restatable}{proposition}{independentFunctions}
    \label{prop:independentfunctions}
    Let \(\Theta \subset \R^D\) be open, and let \(H = (h_1, \ldots, h_k): \Theta \to \R^k\), \(h: \Theta \to \R\) be \(C^r\) functions for \(r \geq 1\). Assume \(\nabla_\theta h_1, \ldots, \nabla_\theta h_k\) are linearly independent at every \(\theta \in \Theta\). Then the following two conditions are equivalent:
    \begin{enumerate}[label = (\roman*)]
        \item For every \(\theta \in \Theta\),  \(\nabla_\theta h(\theta) \in \operatorname{span}\{\nabla_\theta h_1(\theta), \ldots, \nabla_\theta h_k(\theta)\}\).
        \item For every \(\theta_0 \in \Theta\), there 
        exist neighborhoods 
        \(V \ni \theta_0, W \ni H(\theta_0)\) and \(f \in C^r(W)\) such that 
        \begin{equation} \label{eq:functional_dependence}
            h(\theta) = f(H(\theta)) = f(h_1(\theta), \ldots, h_k(\theta)), \quad \forall \theta \in V.
        \end{equation}
    \end{enumerate}
\end{restatable}

\subsection{Loss symmetries and functional symmetries}

A symmetry of an object is an invertible transformation that leaves the structure of the object unchanged. 
As defined in the
survey by \cite{zhao_symmetry_2025}, a bijective parameter transformation \(T: \Theta \to \Theta\) is a loss symmetry if it preserves the value of the loss function \(L\).
It is a functional symmetry if it preserves the function represented by the model \(G\).
Every functional symmetry is a loss symmetry, but the converse is not true in general.

Since symmetries are invertible transformations, they form a group under composition.
Hence, prior works have utilized the language of group actions to study symmetries.
Let \(\Gamma\) be a group that acts on \(\Theta\) via a group action \(\psi: \Gamma \times \Theta \to \Theta\) satisfying the group action axioms:
\begin{align*}
\psi(e, \theta) &= \theta, \qquad \forall \theta \in \Theta, \text{ where $e$ is the identity element in }\Gamma,\\
\psi(g, \psi(h, \theta)) &= \psi(gh, \theta), \quad \forall g, h \in \Gamma, \forall \theta \in \Theta.
\end{align*}
For convenience, we also write the group action as \(g \cdot \theta := \psi(g, \theta)\).

For a given dataset \(\mathcal D\), we say that \((\Gamma, \psi)\) is a symmetry of the loss function \(L_{\mathcal D}\) if, for all \(g \in \Gamma, \theta \in \Theta\), 
\[
L_{\mathcal D}(g, \theta) := L_{\mathcal D}(\psi(g, \theta)) = L_{\mathcal D}(g \cdot \theta) = L_{\mathcal D}(\theta).
\] 
If the symmetry holds for all datasets \(\mathcal D\), we simply say that \((\Gamma, \psi)\) is a symmetry of the loss function \(L\). 
Similarly, we say that \((\Gamma, \psi)\) is a symmetry of the model \(G\) if, for all \(g \in \Gamma, \theta \in \Theta, x \in \mathcal X\), 
\[
\mathcal G(g, \theta, x) := G(\psi(g, \theta), x) = G(g \cdot \theta, x) = G(\theta, x).
\]

Of particular interest are differentiable symmetries, where \(\Gamma\) is a Lie group and the group action \(\psi\) is smooth.
If \(\Gamma\) is a one-dimensional connected Lie group, it is well-known that \(\Gamma\) is isomorphic to the real line \(\R\) or the circle \(S^1\). 
But since \(\R\) is a universal cover of both groups, we can always lift the group action of \(\Gamma\) to one of \(\R\). Thus, it suffices to consider \(\Gamma = \R\). 
In this case, \cite{kunin2021neuralmechanicssymmetrybroken} derive  the following necessary condition for \((\R,\psi)\) to be a differentiable symmetry of \(L_{\mathcal D}\): 
\begin{equation}
    \label{eqn:kunin_condition}
    \langle \nabla L_{\mathcal D}(\theta), D_g \psi(0, \theta) \rangle  = 0, \quad \forall \theta \in \Theta.
\end{equation}

\subsection{A gap between differentiable symmetries and conservation laws} 
    \label{sec:motivating_example} 
    The following example illustrates the gap between the number of independent differentiable symmetries and independent conservation laws.
    Consider a two-layer linear network 
    \[
    G: (\R^{2 \times 2})^2 \times \R^2 \to \R^2, \quad G((U,V), x) = UV^\top x.
    \]
    By Corollary 4.4 of \cite{marcotte2024abidelawfollowflow}, a complete set of conservation laws for this network is 
    \begin{align*}
        h_1(U,V) &= U_{11}^2 + U_{21}^2 - V_{11}^2 - V_{21}^2,\\
        h_2(U,V) &= U_{11}U_{12} + U_{21}U_{22} - V_{11}V_{12} - V_{21}V_{22},\\
        h_3(U,V) &= U_{12}^2 + U_{22}^2 - V_{12}^2 - V_{22}^2.
    \end{align*}
    It is also known that for any invertible matrix \(S \in \R^{2 \times 2}\), the transformation \((U,V) \mapsto (US, V(S^{-1})^\top)\) is a functional symmetry of the network. 
    In particular, the following maps \(\psi_i: \R \times (\R^{2 \times 2})^2 \to (\R^{2 \times 2})^2\), \(i = \overline{1, 4}\), give four one-parameter symmetries of the network, whose infinitesimal generators \(D_t \psi_i(0, \theta), i = \overline{1, 4},\) are linearly independent vector fields: 
    \begin{align*}
        \psi_1 = \left(U \begin{psmallmatrix}
            e^{2t} & 0 \\ 0 & 1
        \end{psmallmatrix}, V \begin{psmallmatrix}
            e^{-2t} & 0 \\ 0 & 1
        \end{psmallmatrix}\right),
        & \quad \psi_2 = \left(U \left(\begin{smallmatrix}
            \cosh(t) & \sinh(t) \\ \sinh(t) & \cosh(t)
        \end{smallmatrix}\right), V \left(\begin{smallmatrix}
            \cosh(t) & -\sinh(t) \\ -\sinh(t) & \cosh(t)
        \end{smallmatrix}\right)\right),\\
        \psi_3 = \left(U \begin{psmallmatrix}
            1 & 0 \\ 0 & e^{2t}
        \end{psmallmatrix}, V \begin{psmallmatrix}
            1 & 0 \\ 0 & e^{-2t}
        \end{psmallmatrix}\right),
        &\quad \psi_4 = \left(U \left(\begin{smallmatrix}
            \cos(t) & -\sin(t) \\ \sin(t) & \cos(t)
        \end{smallmatrix}\right), V \left(\begin{smallmatrix}
            \cos(t) & -\sin(t) \\ \sin(t) & \cos(t)
        \end{smallmatrix}\right)\right).
    \end{align*}

    Observe that \(\nabla_\theta h_i(\theta) = D_t \psi_i(0, \theta)\) for each \(i = 1, 2, 3\).
    Since these conservation laws are complete, \(\psi_4\) 
    yields no new independent conservation law. 
    This runs counter to a naive analogy with Noether's theorem, which might suggest that every differentiable symmetry gives rise to a corresponding conservation law. 

    Several works have alluded to this gap in different contexts, including those of \cite{zhao_symmetries_2023, gegenfurtner2026symmetries}.
    The recent survey by \cite{zhao_symmetry_2025} explicitly 
    identifies as an open question whether all conserved quantities originate from known symmetries. 
    We formalize this gap between differentiable symmetries and conservation laws in Section \ref{sec:general}.

\section{Loss Symmetries, Conservation Laws, and Vector Fields} 
\label{sec:general}

In this section, we build on standard results from differential geometry to develop a framework to characterize loss symmetries and conservation laws in neural networks under gradient flow via their underlying vector fields.
We provide a short review of the relevant concepts in Appendix \ref{app:diffgeo} for completeness, and refer to standard texts \citep{olver_applications_nodate,isidori1995nonlinear,lee2003introduction} for further details. 

Our first result establishes a necessary and sufficient condition for a Lie group action to be a differentiable symmetry of a given loss function, thus both extending and generalizing the necessary condition (\ref{eqn:kunin_condition}) derived by \cite{kunin2021neuralmechanicssymmetrybroken}.

\begin{restatable}{proposition}{symmetryEquivalence}\label{prop:symmetry_equivalence} Let \(\Gamma\) be a connected Lie group that acts on \(\Theta\) via a smooth group action \(\psi\). 
    A necessary and sufficient condition for \(\Gamma\) to be a differentiable symmetry of \(L_{\mathcal D}\) is
    \begin{equation} \label{eq:symmetry_condition}
        D_g \psi(e, \theta)^\top \nabla_\theta L_{\mathcal D}(\theta) = 0, \quad \forall \theta \in \Theta.
    \end{equation}
\end{restatable}

A direct application of Proposition \ref{prop:symmetry_equivalence} yields Corollary \ref{cor:1dsymmetrycond} in the case where \(\Gamma = \R\).
This then further yields Corollary \ref{cor:1dsymmetrycond2} by imposing that the symmetry holds for all datasets \(\mathcal D \in (\mathcal X \times \mathcal Y)^\ast.\)
\begin{restatable}{corollary}{oneDimensionalSymmetryCondition}
\label{cor:1dsymmetrycond}
    Let \(\Gamma = \R\) be the additive group of real numbers that acts on \(\Theta\) via a smooth group action \(\psi(t, \theta)\). 
    Let \(\mathcal D\) be a dataset in \((\mathcal X \times \mathcal Y)^\ast\).
    A necessary and sufficient condition for \(\Gamma\) to be a differentiable symmetry of \(L_{\mathcal D}\) is
    \begin{equation}\label{eq:symmetry_equivalence_one_parameter}
        \langle D_t \psi(0, \theta) ,\nabla_\theta L_{\mathcal D}(\theta)\rangle = 0, \quad \forall \theta \in \Theta.
    \end{equation}
\end{restatable}

\begin{restatable}{corollary}{globalOneDimensionalSymmetryCondition}
    \label{cor:1dsymmetrycond2}
    Let \(\Gamma = \R\) be the additive group of real numbers that acts on \(\Theta\) via a smooth group action \(\psi(t, \theta)\). 
    A necessary and sufficient condition for \(\Gamma\) to be a differentiable symmetry of \(L\) is
    \begin{equation}\label{eq:symmetry_equivalence_one_parameter2}
        D_t \psi(0, \theta) \perp W_\theta, \quad \forall \theta \in \Theta.
    \end{equation}
\end{restatable}

We remark that the smooth action of any connected Lie group \(\Gamma\) is uniquely determined by the smooth actions of its \(1\)-parameter Lie subgroups (see Appendix \ref{sec:1d_justification}). 
Thus, \(1\)-parameter symmetries provide the building blocks for higher dimensional symmetries, motivating us to focus on those. 

Since Proposition \ref{prop:conservation-law-condition} and Corollary \ref{cor:1dsymmetrycond2} are both equivalent conditions of the same form, they together suggest that conservation laws and symmetries can be completely characterized by smooth vector fields that are perpendicular to the gradient flow distribution \(W\), i.e., those in \(W^\perp\). 
Indeed, given enough regularity, we fully describe the connection between vector fields in \(W^\perp\) and the corresponding conservation laws and symmetries via Proposition~\ref{prop:conservation_law_vector_field} and Proposition~\ref{prop:symmetry_vector_field}. 
Before doing so, we recall several notions from differential geometry. 
A vector field \(\chi = (\chi_1, \dots, \chi_m): \Theta \to \R^D\) is said to be closed if its cross partial derivatives are equal, i.e., 
\begin{equation}
    \label{eqn:closed_field}
    \frac{\partial \chi_i}{\partial \theta_j} = \frac{\partial \chi_j}{\partial \theta_i}, \quad \forall 1 \leq i, j \leq D.
\end{equation}
It is
conservative, or exact, if it is the gradient of some scalar potential function, i.e., \(\chi = \nabla h\) for some \(h: \Theta \to \R\). 
A vector field is
complete if the flow generated by it exists for all time.

\begin{restatable}{proposition}{conservationLawVectorField}
    \label{prop:conservation_law_vector_field}
    If \(h\) is a conservation law, then \(\chi := \nabla h\) is a closed vector field in \(W^\perp\). 
    On the other hand, if \(\chi\) is a closed vector field in \(W^\perp\), then locally around every \(\theta \in \Theta\), there exists a unique scalar potential function \(h\) such that \(\chi = \nabla h\), up to an additive constant. 
    Furthermore, if \(\Theta\) is a star-shaped domain, then the scalar potential function \(h\) exists globally on \(\Theta\).
\end{restatable}

\begin{restatable}{proposition}{symmetryVectorField}
    \label{prop:symmetry_vector_field}
    If \((\R, \psi)\) is a symmetry of the loss function \(L\), then \(\chi := D_t \psi \mid_{t = 0}\) is a smooth vector field in \(W^\perp\).
    On the other hand, if \(\chi \in W^\perp\) is a smooth vector field, then there exists a unique maximal local flow \(\psi: \mathcal O \to \Theta\), where \(\mathcal O \subset \R \times \Theta\) is the maximal open set containing \(\{0\} \times \Theta\), for which the solution of 
    \[\frac{d}{dt} \psi(t, \theta) = \chi(\psi(t, \theta)), \quad \psi(0, \theta) = \theta\]
    is well-defined. 
    This unique maximal local flow satisfies the following four equations for all \(\theta \in \Theta\) and \(t,s\) whenever the terms are well-defined:
    \begin{align*}
        &\psi(0, \theta) = \theta, \qquad 
        \psi(t + s, \theta) = \psi(t, \psi(s, \theta)), &\text{(group action axioms)}\\
        &D_t \psi(0, \theta) = \chi(\theta), &\text{(infinitesimal generator)}\\
        &L(\psi(t, \theta)) = L(\theta). &\text{(loss invariance condition)}
    \end{align*}
    Furthermore, if \(\chi\) is complete, then \(\psi\) is a symmetry of the loss function \(L\).
\end{restatable}

The maximal local flow \(\psi\) as defined in Proposition \ref{prop:symmetry_vector_field} is known as a partial group action \citep{cheeger1986collapsing1,cheeger1990collapsing2,Batista2017partial}, meaning that its domain is a maximal neighborhood of \(\{0\} \times \Theta\), and the group action axioms are satisfied wherever it is defined. 
Hence, we call \((\R, \psi)\) a partial symmetry of the loss function \(L\). 
Since the infinitesimal symmetry condition \ref{cor:1dsymmetrycond2} is a local condition at \(t = 0\), Corollary \ref{cor:1dsymmetrycond}, Corollary \ref{cor:1dsymmetrycond2} and Proposition \ref{prop:symmetry_vector_field} still hold verbatim when applied to partial symmetries (see Appendix \ref{app:partial_symmetry_extension} for details).
The results can be summarized as follows:
\begin{itemize}
    \item There exists a one-to-one correspondence between vector fields in \(W^\perp\) and $1$-parameter partial symmetries of the loss function \(L\).
    The partial symmetry is a global symmetry precisely when its generating vector field is complete.
    \item There exists a one-to-one correspondence between closed vector fields in \(W^\perp\) and local conservation laws of the gradient flow dynamics, up to an additive constant. 
    The local conservation law is a global one when the parameter space \(\Theta\) is star-shaped.
\end{itemize} 
We term \(W^\perp\) the \emph{symmetry distribution}, as it captures the directions in which the parameters can be transformed without changing the loss value. 
The following result gives the exact correspondence between symmetries and conservation laws.
\begin{restatable}{corollary}{lawInducesSymmetry}
    \label{cor:law_induces_symmetry}
    Let \(h: \Theta \to \R\) be a conservation law of the gradient flow dynamics.
    There exists a unique partial symmetry \((\R, \psi)\) of the loss function \(L\) such that
    \[
        D_t \psi(0, \theta) = \nabla h(\theta), \quad \forall \theta.
    \]
    Conversely, let \((\R, \psi)\) be a partial symmetry of the loss function \(L\). If \(D_t \psi(0, \theta)\) is closed, then locally, there exists a unique conservation law \(h\), up to an additive constant, such that 
    \[
        \nabla h(\theta) = D_t \psi(0, \theta), \quad \forall \theta.
    \]
\end{restatable}

The additional closedness assumption above suggests that there are more symmetries than conservation laws. 
To formalize this intuition, we introduce the 
following notions of infinitesimally independent and complete partial symmetries. 
The idea is to define a set of partial symmetries such that their infinitesimal actions span the symmetry distribution \(W^\perp\).

\begin{definition}
    A family of \(r\) partial symmetries \(\{\psi_1, \ldots, \psi_r\}\) of \(\R\) onto \(\Theta\) is said to be infinitesimally independent if the vectors \(\{D_t \psi_1(0, \xi), \ldots, D_t \psi_r(0, \xi)\}\) are linearly independent for all \(\theta \in \Theta\). 
    They are called infinitesimally complete if they are independent and their infinitesimal actions span the symmetry directions near \(\theta\). 
\end{definition}
 
We recall Theorem 3.3 of \citet{marcotte2024abidelawfollowflow}, which states that if the dimension of the Lie completion of the reparametrized gradient flow distribution \(\operatorname{dim}(\operatorname{Lie}(W_\phi)(\theta))\) is locally constant, 
then there are exactly \(D - \operatorname{dim}(\operatorname{Lie}(W_\phi)(\theta))\) independent conservation laws in a neighborhood of \(\theta\). 
The following
theorem generalizes this result by removing the need for a nice reparametrization of the parameters, as well as extending it to account for symmetries of the loss function on the regular rank strata of the gradient flow distribution.

\begin{restatable}{theorem}{lawSymmetryGap}
    \label{thm:law-symmetry-gap}
    Assume that the gradient flow distribution \(W\) is a smooth distribution of constant rank \(r\), and the Lie completion \(\operatorname{Lie}(W)\) is a smooth distribution of constant rank \(\overline r\) in \(\Theta\).
    Then, for every \(\theta \in \Theta\), there exists a neighborhood \(U \ni \theta\) such that the following statements hold:
    \begin{enumerate}[label = (\roman*)]
        \item There exists \(s := D - r\) partial symmetries of the loss function \(L\) that are independent on \(U\), and they are complete.
        \item There exist \(c := D - \overline r\) independent conservation laws of the gradient flow dynamics on \(U\), and they are complete. 
        \item The symmetry-conservation law gap is 
        \(s - c := \overline r - r \geq 0.\)
        Equality holds if and only if \(W\) is involutive, i.e., \(\operatorname{Lie}(W) = W\).
    \end{enumerate}
\end{restatable}

One consequence of Theorem~\ref{thm:law-symmetry-gap} is that, if \(W\) is not involutive, then the conservation laws at initalization can not 
determine the specific local minimum that the gradient flow dynamics might converge to. 
Nevertheless, they still provide an initialization-dependent implicit bias that can be used to determine important properties of this minimum. 
We demonstrate these ideas for a two-layer linear network in the simple regression task of learning a scaling function in Appendix \ref{app:learning_to_scale}, where we show that 
the singular values and condition number of the learned weight matrices are determined by the conservation laws at initialization.

\section{Inheritance Principle for Symmetries and Conservation Laws} 
\label{sec:multilayer}

While prior work has made significant progress in understanding the symmetries and conservation laws 
in individual layers, 
a fundamental challenge has been to understand how these are inherited by a larger network that contains the subnetwork as a component, 
and whether the inherited symmetries and conservation laws are complete. 
In this section, we connect the framework developed in Section~\ref{sec:general} to functional symmetry and parameter identifiability, and obtain general inheritance and completeness results for symmetries and conservation laws of multilayer networks.

Recall that every functional symmetry of the model \(G\) is also a symmetry of the loss function \(L\). 
If the loss function \(L\) is sufficiently well-behaved, then the converse is also true, i.e., every loss symmetry of \(L\) is also a functional symmetry of \(G\). 
One way to formalize this notion is through the following Assumption \ref{as:separability}, which we term the \textit{separability of predictions} property.
This is a natural condition that is satisfied by many loss functions (see Appendix \ref{app:loss_assumption}). 
Based on this, we obtain Proposition \ref{prop:separability-symmetry}, which allows us to utilize the framework developed in the previous section to obtain our inheritance and completeness results.
\begin{restatable}[Separability of predictions]{assumption}{separability} 
    \label{as:separability}
    The loss function \(\ell: \mathcal Z \times \mathcal Y \to \R\) satisfies
    \[\ell(z, y) = \ell(z', y), \forall y \in \mathcal Y \implies z = z'.\]
\end{restatable} 

\begin{restatable}{proposition}{separabilitySymmetry}
    \label{prop:separability-symmetry}
    If the loss function \(\ell\) satisfies Assumption \ref{as:separability}, then any partial symmetry of the loss function \(L\) is also a partial functional symmetry of the model \(G\).
\end{restatable}

Our first result, Proposition \ref{prop:inheritance_1}, 
establishes a general inheritance result for symmetries and conservation laws of subnetworks embedded in a larger network. 
Proposition J.3 by \cite{marcotte2025transformativeconservativeconservationlaws} formulated an inheritance result for conservation laws arising from invariant transformations for serial compositions. 
Building on the same underlying principle, we formulate our result directly in terms of preservation of functional equivalence. 
Our result naturally covers a broad class of compositional structures, including serial composition, sum, concatenation, Hadamard product, outer product, graph aggregation, etc. of subnetworks (see Table \ref{tab:inheritance_examples} in Appendix \ref{subsec:proof-inheritance-1} for further details). 

\begin{restatable}[Inheritance from subnetworks I]{proposition}{inheritanceI}
\label{prop:inheritance_1}
Let $g:\Theta\times\mathcal X^g\to\mathcal Z^g$ and
$G:(\Theta\times\Omega)\times\mathcal X\to\mathcal Z$
satisfy
\begin{equation}
    \label{eqn:indentifiability_1}
    g(\theta,\cdot)=g(\theta',\cdot)
    \quad\Longrightarrow\quad
    G((\theta,\omega),\cdot)=G((\theta',\omega),\cdot)
\qquad\text{for every }\omega\in\Omega.
\end{equation}
Every partial symmetry \((\R, \psi)\) or conservation law \(h\) of $g$ for a loss satisfying Assumption~\ref{as:separability} extends to a partial symmetry \((\R, \tilde \psi)\) or conservation law \(\tilde h\) of $G$ for any loss function via
\begin{equation}
    \label{eqn:inheritance_extension}
    \tilde\psi\left(t, (\theta,\omega)\right)=(\psi(t, \theta),\omega),
    \qquad
    \text{or}
    \qquad
    \tilde h(\theta,\omega)=h(\theta).
\end{equation}
\end{restatable}

Our second result, Proposition \ref{prop:inheritance_2}, shows that if the functional equivalence of the larger network forces the functional equivalence of its subnetworks, then the symmetries and conservation laws of the larger network must be infinitesimally generated by the natural extensions of the symmetries and conservation laws of its subnetworks. 

\begin{restatable}[Inheritance from subnetworks II]{proposition}{inheritanceII}
    \label{prop:inheritance_2}
    Let $g^j:\Theta^j\times\mathcal X^{g^j}\to\mathcal Z^{g^j}, j = 1, 2$ and
    $G:(\Theta^1 \times \Theta^2)\times\mathcal X\to\mathcal Z$
    satisfy
    \begin{equation}
        \label{eqn:identifiability_2}
        [G((\theta_1, \theta_2), \cdot) = G((\theta'_1, \theta'_2), \cdot)] \implies [g^j(\theta_j, \cdot) = g^j(\theta'_j, \cdot), \forall j = 1, 2].
    \end{equation}
    For \(j = 1, 2\), let \((\R, \psi^j_i)_{i=1}^{s_j}\) be a complete set of partial functional symmetries of \(g^j\), and let \(h^j_1, \ldots, h^j_{c_j}\) be a complete set of conservation laws of \(g^j\) 
    w.r.t.\ a loss satisfying Assumption~\ref{as:separability}.

    Every partial symmetry \((\R, \psi)\) or conservation law \(h\) of \(G\) with respect to a loss satisfying Assumption \ref{as:separability} is infinitesimally generated by the extensions 
    \[\tilde \psi^1_i(t, (\theta_1, \theta_2)) = (\psi^1_i(t, \theta_1), \theta_2), 
    \quad
    \tilde \psi^2_i(t, (\theta_1, \theta_2)) = (\theta_1, \psi^2_i(t, \theta_2)),
    \quad
    \text{or}
    \quad 
    \tilde h^j_i(\theta_1, \theta_2) = h^j_i(\theta_j),\]
    respectively.
    The extensions themselves need not be partial symmetries nor conservation laws of \(G\).
\end{restatable}

Finally, by combining the previous two results, we obtain a completeness result for symmetries and conservation laws of multilayer networks.

\begin{restatable}[Inheritance from subnetworks III]{theorem}{inheritanceIII}
     \label{thm:inheritance_3}
    Let $g^j:\Theta^j\times\mathcal X^{g^j}\to\mathcal Z^{g^j}, j = 1, \ldots, k$ and
    $G:(\Theta^1 \times \ldots \times \Theta^k)\times\mathcal X\to\mathcal Z$
    satisfy
    \begin{equation}
        \label{eqn:identifiability_3}
        [G((\theta_1, \ldots, \theta_k), \cdot) = G((\theta'_1, \ldots, \theta'_k), \cdot)] \iff [g^j(\theta_j, \cdot) = g^j(\theta'_j, \cdot), \forall j = 1, \ldots, k].
    \end{equation}
    For \(j = \overline{1, k}\), let \((\R, \psi^j_i)_{i=1}^{s_j}\) be a complete set of partial functional symmetries of \(g^j\), 
    and let \(h^j_1, \ldots, h^j_{c_j}\) be a complete set of conservation laws of \(g^j\) 
    w.r.t.\ a loss satisfying Assumption~\ref{as:separability}. 
    
    Assume that the loss used for the larger network \(G\) also satisfies Assumption \ref{as:separability}. 
    A complete set of partial functional symmetries of \(G\) is given by the extensions
    \[
    \tilde \psi^j_i(t, (\theta_1, \ldots, \theta_k)) = (\theta_1, \ldots, \psi^j_i(t, \theta_j), \ldots, \theta_k), \qquad j = \overline{1, k},\; i = \overline{1, s_j},
    \]
    and a complete set of conservation laws of \(G\) is given by
    \[\tilde h^j_i(\theta_1, \ldots, \theta_k) = h^j_i(\theta_j), \qquad j =\overline{1, k},\; i = \overline{1, c_j}.
    \]
\end{restatable}

We call \eqref{eqn:identifiability_3} the \emph{compositional identifiability} condition. 
The implication is 
that if the larger network is compositionally identifiable, then all of its symmetries and conservation laws are inherited from its component subnetworks.

\section{Application to Neural Networks}
\label{sec:application}

We demonstrate the utility of our results by applying them to several neural network architectures.

The first is grouped query attention (GQA) \citep{ainslie2023gqa}, which is widely used in modern transformer architectures for fast inference with KV caching.
\cite{marcotte2025transformativeconservativeconservationlaws} posed completeness of the conservation laws of the standard multi-head self-attention (MHSA) mechanism \citep{vaswani2023attentionneed} 
as an open problem. 
Later, \cite{tran2026conservation} provided a 
formulation and proof of this result. 
However, the proof 
does not provide an explicit justification that, near regular points, every solution of the derived PDE system is locally a function of the known conservation laws. 
Thus, the argument as written leaves the local completeness step unestablished. 

Using our framework, we provide a proof for the complete characterization of the symmetries as well as conservation laws of the GQA mechanism in Proposition \ref{prop:grouped-query-attention}, which implies completeness for the MHSA mechanism as a special case.
The proof is short and instructive, relying on results on completeness for two-layer linear networks \citep{marcotte2024abidelawfollowflow} and identifiability of attention heads \citep{tran2025equivariantneuralfunctionalnetworks}. 

\begin{restatable}[Grouped query attention]{proposition}{groupedQueryAttention}
    \label{prop:grouped-query-attention}
    Let \(n_{H} = k n_{G}\) and \(D_h \leq D\) be positive integers.
    Consider the grouped query attention mechanism \(G\) with \(n_H\) heads and \(n_G\) groups with parameters
    \[\theta = (Q^j_1, \ldots, Q^j_{k}, K^j, V^j, O^j_1, \ldots, O^j_k)_{j=1}^{n_G} \in \left(\left(\R^{D \times D_h}\right)^{k} \times \R^{D \times D_h} \times \R^{D \times D_h} \times \left(\R^{D \times D_h}\right)^{k}\right)^{n_G},\]
    where for any \(X \in \cup_{L = 1}^{\infty} \R^{L \times D}\),
    \[G(\theta, X) = \sum_{j=1}^{n_G} \sum_{i=1}^{k} \operatorname{softmax}\left(\frac{X Q^j_i (K^j)^\top X^\top}{\sqrt{D_h}}\right) X V^j (O^j_i)^\top.\]
    Let \(\theta_0\) be a point in the parameter space so that \(Q^j_i, K^j, V^j, O^j_i\) are full rank and the matrices \(Q^j_i (K^j)^\top\) are pairwise distinct for \(j = \overline{1, n_G}, i = \overline{1, k}\).
    Locally around \(\theta_0\), the following hold:
    \begin{itemize}
    \item All functionally equivalent parameters of the grouped query attention mechanism \(G\) are given by the following partial group action of \((\operatorname{GL}_{D_h}(\R) \times \operatorname{GL}_{D_h}(\R))^{n_G}\):
    \[
        (S^j, T^j)_{j=1}^{n_G} \cdot (Q^j_1, \ldots, Q^j_{k}, K^j, V^j, O^j_1, \ldots, O^j_k)_{j=1}^{n_G} 
        = (\tilde Q^j_1, \ldots, \tilde Q^j_{k}, \tilde K^j, \tilde V^j, \tilde O^j_1, \ldots, \tilde O^j_k)_{j=1}^{n_G},
    \]
    where
    \[
        \tilde Q^j_i = Q^j_i S^j, \quad 
        \tilde K^j = K^j (S^j)^{- \top}, \quad
        \tilde V^j = V^j T^j, \quad 
        \tilde O^j_i = O^j_i (T^j)^{-\top}.
    \]
    \item 
    The upper triangular elements of the following matrix functions form a complete set of conservation laws with respect to any loss satisfying Assumption \ref{as:separability}:
    \[\sum_{i=1}^{k} (Q_i^j)^\top Q^j_i - (K^j)^\top K^j, \quad (V^j)^\top V^j - \sum_{i=1}^k (O^j_i)^\top O^j_i, \quad j \in \overline{1, n_G}.\]
    \end{itemize}
    The case 
    \(n_{H} = n_{G}\) corresponds to 
    standard multi-head self-attention, 
    and all results still hold.
\end{restatable}

In Appendix \ref{app:exp2}, we examine the conservation laws derived in Proposition \ref{prop:grouped-query-attention} and their drift during prolonged constant step size SGD in a language model. 
We observe that approximate conservation can coexist with learning,
as reflected in reductions in training and validation loss. 
However, we also observe that conservation law preservation for certain parameter blocks may deteriorate during prolonged SGD. 

Next, we consider polynomial neural networks (PNNs). 
Recently, 
\cite{usevich2025identifiability} showed that a wide variety of polynomial networks satisfy a finite-to-one identifiability condition. 
Our result relies on showing that for these networks, the symmetry-conservation law gap is zero, allowing for a direct transfer of symmetries to conservation laws. 
This is the first deep network architecture that we know of for which a complete characterization of the conservation laws has been obtained.

\begin{restatable}[Polynomial neural network]{proposition}{PNN}
    \label{prop:PNN}
    Consider the polynomial neural network (PNN) mechanism \(G\) with parameters
    \(\theta = (W_j, b_j)_{j=1}^L \in \prod_{j=1}^L \left(\R^{D_j \times D_{j-1}} \times \R^{D_j}\right)\)
    defined by 
    \[G(\theta, x) = f_L \circ \rho_{r_{L-1}} \circ f_{L-1} \circ \cdots \circ \rho_{r_1} \circ f_1(x),\]
    where \(\rho_{r_j}(z) = (z_1^{r_j}, \ldots, z_{D_j}^{r_j})\) is a monomial activation function of degree \(r_j\) and \(f_j(z) = W_j z + b_j\) for \(j = 1, \ldots, L\).
    If the architecture is finite-to-one at a generic \(\theta_0\), then locally around \(\theta_0\), a complete set of conservation laws of \(G\) with respect to any loss satisfying Assumption \ref{as:separability} is given by
        \[h^j_i(\theta) = \|(W_j)_{i,:}\|_2^2 + (b_j)_{i}^2 - r_j \|(W_{j+1})_{:,i}\|_2^2, \quad i = \overline{1, D_j}, j = \overline{1, L-1}.\]
\end{restatable}

In the deep square linear network case, \cite{lindsey2026regularizationimpliesbalancednessdeep} recently characterized the symmetries of the model on the full-rank strata of the parameter space. 
Theorem \ref{thm:law-symmetry-gap} allows us to obtai an upper bound of \((L-1)D^2\) independent conservation laws of the model, which improves slightly upon the obvious bound of \(LD^2\). 
More generally, this
illustrates that a complete characterization of the symmetries of a model automatically yields an upper bound on the number of independent conservation laws, which may be useful for future work on other architectures. 

\begin{restatable}
[Deep square linear network]{proposition}{deepLinearNetwork}
    \label{prop:deep_linear_network}
    Consider the deep square linear network mechanism \(G\) with parameters
    \(\theta = (W_j)_{j=1}^L \in \prod_{j=1}^L \R^{D \times D}\)
    defined by 
    \[G(\theta, x) = W_L W_{L-1} \ldots W_1 x.\]
    If \(W_L \ldots W_1\) is full rank, then locally around \((W_1, \ldots, W_L)\), there are at most \((L-1) D^2\) independent conservation laws of \(G\) with respect to any loss satisfying Assumption \ref{as:separability}.
\end{restatable}

Finally, we observe that Proposition \ref{prop:inheritance_1} readily recovers known conservation laws derived by prior symmetry-based frameworks. 
For example, if trainable weight matrices \(A, B\) enter a multilayer network only through their product \(AB\), as in the popular low-rank adaptation (LoRA) scheme \citep{hu2022lora}, then it automatically inherits conservation laws given by the entries of \(A^\top A - BB^\top\).
We refer the reader to Appendix \ref{app:comparison_sub} for more details.

\section{Conclusion} 

We develop a unified geometric framework for studying symmetries and conservation laws in neural network gradient flow. 
It clarifies the gap between symmetries and conservation laws, establishes results on their inheritance from subnetworks and their completeness in larger networks, and characterizes the complete set of conservation laws for several architectures. 
A limitation is that compositional identifiability does not hold for matrix multiplication and thus our inheritance result cannot be applied to deep linear networks. 
Future work could pursue inheritance under weaker compositional identifiability conditions. 
Other promising directions include extending our results to 
non-gradient-flow and data dependent dynamics, 
and investigation of conservation-laws informed initialization and training.

\subsection*{Acknowledgments} 
We thank Viet-Hoang Tran for insighful discussions on identifiability and conservation laws in self-attention modules, and Ang\'elica Torres for discussions on the completeness of conservation laws in deep linear networks. 
KN thanks Trung-Nghia Nguyen for discussions on vector fields and flows.
GM thanks the Institute for Mathematical Sciences, National University of Singapore, for its hospitality during a visit. 

This work was supported in part by NSF grant CCF-2212520 through the Mathematical and Scientific Foundations of Deep Learning (MoDL) program. 
GM was also supported in part by DARPA through the Artificial Intelligence Quantified (AIQ) program under grant HR00112520014; NSF grants DMS-2145630, DMS-2522495, and RTG 2446222; DFG project 464109215 within Priority Programme SPP 2298 ``Theoretical Foundations of Deep Learning''; and the BMFTR through DAAD project 57616814 (SECAI).

\subsection*{AI use statement}

The theoretical framework and essential proof arguments were developed by the authors without generative AI assistance. 

We used generative AI tools to help formulate mathematical claims, propose or refine hypotheses, provide feedback on research methodology and experiments, assist with implementation and interpretation of experimental results, and suggest experimental parameters. We also used these tools to create or edit code and figures, identify and summarize relevant literature, brainstorm research directions, and improve the readability of the manuscript.

We did not use generative AI to develop theoretical models or conceptual frameworks, supply critical ingredients for proofs, generate synthetic datasets,
assist in the writing of proofs, clean or reformat datasets, or translate text. The authors reviewed all AI-assisted material, independently checked the mathematical content and cited literature, and verified and tested AI-assisted code. We take responsibility for the final content of this work. 

\newpage 
\bibliographystyle{unsrtnat}
\bibliography{main_arxiv}

\newpage

\appendix 

\section{Background on Vector Fields}

\label{app:diffgeo}

We present a brief overview of the necessary background on vector fields and differential geometry based on Chapter 1 of \cite{isidori1995nonlinear}. For a more comprehensive treatment, we refer the reader to \cite{isidori1995nonlinear, olver_applications_nodate,lee2003introduction}.

We consider functions on an open subset \(U \subset \R^D\), where \(U\) can be taken as \(\R^D\) itself in the globally smooth regime. 
A \textit{vector field} is a smooth function \(f: U \to \R^D\) that associates to each $\theta \in U$ a direction in \(\R^D\). A \textit{covector field} is a smooth function \(w^\ast : U \to (\R^D)^\ast\), where we identify the dual space \((\R^D)^\ast = \{T: \R^D \to \R, T \text{ linear}\}\) with the space of row vectors \(\R^{1 \times D}\). This makes it convenient to write 
\[w^\ast(f) = w^\ast f = \langle w, f \rangle, \quad w= (w^\ast)^\top : U \to \R^D.\]
Let \(\lambda: U \to \R\) be a smooth function, then its derivative is the covector field
\[D \lambda (\theta) = \left(\frac{\partial \lambda}{\partial \theta_1}, \ldots, \frac{\partial \lambda}{\partial \theta_n}\right).\]
Given a vector field \(f\), the \textit{derivative of \(\lambda\) along \(f\)} is written as \(L_f\lambda(\theta)\), and is defined to be the inner product between \(\nabla \lambda = (D \lambda)^\top\) with \(f\)
\[L_f\lambda(\theta) = D \lambda (\theta) f(\theta) = \langle \nabla \lambda(\theta), f(\theta) \rangle = \sum_{i=1}^n \frac{\partial \lambda}{\partial \theta_i} f_i(\theta).\]

Given two prior vector fields \(f, g\). The \textit{Lie bracket} \([f,g]\) is a vector field satisfying
\[[f,g](\theta) = Dg(\theta) \, f(\theta) - Df(\theta) \, g(\theta).\]
The Lie bracket is bilinear over \(\R\), skew commutative, and satisfies the Jacobi identity.

Let \(f_1, \ldots, f_d\) be smooth vector fields. These identify a unique \textit{smooth distribution} given by
\[\Delta(\theta) = \operatorname{span}\{f_1(\theta), \ldots, f_d(\theta)\},\]
which associates each point \(\theta \in U\) with a subspace of \(\R^n\) spanned by all \(f_i\) at \(\theta\). The term `smooth' here refers to the fact that it can be spanned by smooth vector fields. 
A vector field \(f\) belongs to a distribution \(\Delta\) if this is true pointwise. The dimension of a distribution \(\Delta\) at \(\theta \in U\) is \(\operatorname{dim} \Delta(\theta)\). If \(F\) is a smooth matrix with \(n\) rows then the distribution spanned by its columns is 
\[\Delta(\theta) = \operatorname{Im}(F(\theta)), \quad \operatorname{dim} \Delta(\theta) = \operatorname{rank} F(\theta).\]

A distribution \(\Delta\) is called \textit{nonsingular} if its dimension \(\operatorname{dim} \Delta(\theta) \equiv d\) is constant on \(U\). A point \(\theta^0\) is said to be a \textit{regular point} if there exists a neighborhood \(U^0\) on which \(\Delta\) is nonsingular. If \(\Delta\) is a smooth, nonsingular distribution of dimension \(d\) then it has a spanning set of vector fields \(f_1, \ldots, f_d\) such that \(\Delta(\theta) = \operatorname{span}\{f_i(\theta)\}\) and any smooth vector field \(\tau \in \Delta\) can be expressed as \(\tau(\theta) = \sum c_i(\theta) f_i(\theta)\) for smooth real-valued functions \(c_i\).
A distribution \(\Delta\) is \textit{involutive} if it is closed under taking the Lie bracket, i.e., \([\tau_1, \tau_2] \in \Delta, \forall \tau_1, \tau_2 \in \Delta\). 

Similar to a distribution, given smooth covector fields \(w^\ast_1, \ldots, w^\ast_d\), we can define the \textit{smooth codistritbution} \(\Theta = \operatorname{span}(w_1^\ast, \ldots, w_d^\ast)\). Similar constructions to a distribution applies. 

There is a natural way to construct a codistribution starting from a distribution \(\Delta\). 
For each \(\theta\), let 
\[\Delta^\perp(\theta) = \{w^\ast \in (\R^n)^\ast: w^\ast v = 0, \forall v \in \Delta(\theta)\}\]
be the annihilator of \(\Delta(\theta)\). This defines a codistribution called the \textit{annihilator} of \(\Delta\). 
Conversely, given a codistribution \(\Theta\), one can similarly construct a distribution \(\Theta^\perp\), called the annihilator of \(\Theta\).
By construction, the pointwise sum of dimensions of \(\Delta\) and \(\Delta^\perp\) is \(n\).

If a distribution \(\Delta\) is spanned by the columns of a smooth matrix \(F\), then its annihilator is identified by 
\[\Delta^\perp(\theta) = \{w^\ast \in (\R^n)^\ast : w^\ast F(\theta) = 0\}.\]
Conversely, if a codistribution \(\Theta\) is spanned by the rows of a smooth matrix \(W\) then its annihilator is identified by
\[\Theta^\perp (\theta) = \{v \in \R^n: W(\theta) v = 0.\}\]
Thus, \(\Theta^\perp(\theta)\) is the kernel of the matrix \(W(\theta)\).

If \(\theta\) is a regular point of a smooth distribution \(\Delta\), then it is also a regular point of \(\Delta^\perp\).
Consider a nonsingular distribution \(\Delta\) with dimension \(d\) and basis vector fields \(f_1, \ldots, f_d\). We know that its codistribution \(\Delta^\perp\) is smooth, nonsingular, and has a basis \(w_1^\ast, \ldots, w_{n-d}^\ast\). By construction, these solve the equation
\[w_j^\ast(\theta) F(\theta) = 0, \quad F(\theta) = (f_1(\theta), \ldots, f_d(\theta))\]
for all \(\theta\). We call \(\Delta\) \textit{completely integrable} if we can choose \(w_1^\ast, \ldots, w_{n-d}^\ast\) such that there exists real-valued smooth functions \(\ell_1, \ldots, \ell_{n-d}\) such that \(D \ell_i = w_i^\ast, \forall i\).

\begin{theorem}[Frobenius]
    A nonsingular distribution is completely integrable if and only if it is involutive.
\end{theorem}

\section{Comparison With Prior Works}
\label{sec:comparison}

\subsection{A typical derviation of Noether's theorem}

Classical Noether theory establishs a connection between symmetries of the Lagrangian and conseration laws of the dynamical system induced by the Lagrangian.
To illustrate this, we briefly review a typical derivation of Noether's theorem in the context of classical mechanics.
Consider the space \(\R^n\) equipped with the Lagrangian 
\[L: \R^n \times \R^n \times \R \to \R, \quad (x, v, t) \mapsto L(x, v, t),\] 
which defines the action functional
\[S[\gamma] = \int_{a}^b L(\gamma(t), \gamma'(t), t) dt\]
for any path \(\gamma:[a,b] \to \R^n\). 
Let \(q\) be an optimal path that minimizes \(S[\gamma]\). 
Then, via calculus of variations, the optimal path \(q\) satisfies the Euler--Lagrange equations
\[\nabla_x L - \frac{d}{dt} \nabla_v L = 0.\]

Let \(\psi: \R \times \R^n \to \R^n\) be a symmetry of the Lagrangian \(L\) and let \(\chi: \R^n \to \R^n\) be the corresponding infinitesimal generator, i.e., 
\[L(\psi(g, q), D_q \psi(g, q) \, \dot q, t) = L(q, \dot{q}, t), \quad \chi(q) = \frac{d}{dg} \psi(g, q)\big|_{g=0},\]
then taking the derivative with respect to \(g\) and evaluating at \(g = 0\) gives the infinitesimal symmetry condition
\[\left \langle \nabla_x L, \chi \right \rangle + \left \langle \nabla_v L, D_q \chi \, \dot q \right \rangle = 0.\]

If we take the derivative of 
\[\lambda = \langle \nabla_v L, \chi \rangle.\]
with respect to time \(t\) then we obtain
\begin{align*}
    \frac{d}{dt} \lambda(t) 
    &= \left \langle \frac{d}{dt} \nabla_v L, \chi \right \rangle + \left \langle \nabla_v L, D_q \chi \, \dot q \right \rangle\\
    &= \underbrace{\left \langle \frac{d}{dt} \nabla_v L - \nabla_x L, \chi \right \rangle}_{=0 \text{ by Euler--Lagrange equations}} + \underbrace{\left \langle \nabla_x L, \chi \right \rangle + \left \langle \nabla_v L, D_q \chi \, \dot q \right \rangle}_{=0 \text{ by the infinitesimal symmetry condition}} = 0
\end{align*}
Thus, we posit that \(\lambda\) is a conserved quantity along the optimal path \(q\).
We call this term the Noether charge associated with the symmetry \(\psi\).
Noether's theorem states that every symmetry of the Lagrangian induces a conservation law taking the form of the Noether charge \(\lambda(t)\).

In the context of gradient flow, it is important to note that the dynamics of the system are not derived from a Lagrangian.
The gradient flow equation \(\dot\theta=-\nabla f(\theta)\) directly specifies the dynamics of the system, and thus Noether's theorem does not directly apply.

\subsection{Comparison with other frameworks connecting symmetries and conservation laws under gradient flow}
\label{app:comparison_sub}

Inspired by Noether's theorem, prior frameworks that connect symmetries and conservation laws have tried to adapt the derivation of Noether's theorem to the context of gradient flow.
Their common approach is to try to integrate the infinitesimal symmetry condition (\ref{eqn:kunin_condition}) to obtain a constant function with respect to time \(t\) along the gradient flow.
We recall this condition here for convenience:
\[\langle \nabla_\theta L_{\mathcal D}(\theta), D_g \psi(0, \theta) \rangle = 0.\]
If this can be factored into a function of \(\theta\) independent of \(t\), then that function is a conserved quantity along the gradient flow.
We now discuss each of these prior works in turn.
 
\cite{kunin2021neuralmechanicssymmetrybroken} notice that if \(\psi\) is a symmetry of the gradient flow \(\dot \theta = - \nabla_{\theta} L_{\mathcal D}(\theta)\) and \(\chi = D_g \psi(0, \theta)\) is its underlying vector field, then
\begin{align*}
    \frac{d}{dt} \langle\theta, \chi(\theta) \rangle  
    &= \frac{d}{dt} \langle\theta, D_g \psi(0, \theta) \rangle\\
    &= \langle \dot \theta, D_g \psi(0, \theta) \rangle + \langle \theta, D_t D_g \psi(0, \theta) \, \dot \theta \rangle\\
    &= \underbrace{-\langle \nabla_\theta L_{\mathcal D}(\theta), D_g \psi(0, \theta) \rangle}_{=0 \text{ by the infinitesimal invariance condition \ref{eqn:kunin_condition}}} - \langle \theta, D_t D_g \psi(0, \theta) \, \nabla_\theta L_{\mathcal D}(\theta) \rangle.
\end{align*}
Hence, if the additional condition
\begin{equation}
    \label{eqn:kunin_condition_2}
    \langle \theta, D_t D_g \psi(0, \theta) \, \nabla_\theta L_{\mathcal D}(\theta) \rangle = 0
\end{equation}
is satisfied, then they posit that
\[\langle\theta, D_g \psi(0, \theta) \rangle\]
is a conserved quantity along the gradient flow.
Thus, not all symmetries can induce conservation laws based on their framework, but only those that satisfy the additional condition \ref{eqn:kunin_condition_2}.
This is reflected in the additional assumption in Theorem 1 of \cite{kunin2021neuralmechanicssymmetrybroken}.

\cite{zhao_symmetries_2023} show that if the infinitesimal generator of a symmetry can be linearly represented by a matrix \(M\), i.e., 
\[\psi(g, \theta) = e^{gM} \theta, \; \text{ and } \; \chi(\theta) = D_g \psi(0, \theta) = M \theta,\]
then the derivative of the same quantity in \cite{kunin2021neuralmechanicssymmetrybroken} can be rewritten as
\[
    \frac{d}{dt} \langle \theta, \chi(\theta) \rangle 
    = \frac{d}{dt} \langle \theta, M \theta \rangle 
    = \langle \dot \theta, M \theta \rangle + \langle \theta, M \dot \theta \rangle
    = \underbrace{\langle \nabla_\theta L_{\mathcal D}(\theta) , D_g \psi(0, \theta) \rangle}_{=0 \text{ by the infinitesimal invariance condition \ref{eqn:kunin_condition}}} + \langle M^\top \theta, \dot \theta \rangle.
\]
Thus, if the additional condition 
\[\langle M^\top \theta, \dot \theta \rangle = 0\]
is satisfied, then \(\langle \theta, M \theta \rangle\) is a conserved quantity along the gradient flow.
In particular, if \(M\) is symmetric or skew-symmetric, then this condition is automatically satisfied as we can convert \(M^\top\) to either \(M\) or \(-M\), yielding the infinitesimal symmetry condition again.
This is reflected in the additional assumption in Proposition 5.1 of \cite{zhao_symmetries_2023}.

To give a concrete example of a network with a symmetry that is associated to a conservation law through our framework, but that does not give rise to a conservation law under the frameworks of \cite{kunin2021neuralmechanicssymmetrybroken,zhao_symmetries_2023}, consider the following scalar network:
    \[
    G(w_1, w_2) = w_1 e^{2 w_2} x.
    \]
    It has a symmetry given by the group \(\R\) acting on the parameters as 
    \[g \cdot (w_1, w_2) \mapsto (w_1 e^{-2g} , w_2 + g).\]
Its infinitesimal generator is 
    \[\chi = D_g(0, (w_1, w_2)) = (-2 w_1, 1),\]
    which automatically falls outside the scope of \cite{zhao_symmetries_2023} since the infinitesimal generator of this symmetry is not linear in the parameters.
The additional condition (\ref{eqn:kunin_condition_2}) in the work of \cite{kunin2021neuralmechanicssymmetrybroken} is not satisfied, as we have
\[
\langle (w_1, w_2), D_t D_g(0, (w_1, w_2)) \, \nabla_{(w_1, w_2)} L_{\mathcal D}(w_1, w_2) \rangle = \langle (w_1, w_2), D_t(-2 w_1, 1) \rangle = -2 w_1 \dot w_1 \neq 0.
\]
In contrast, from our framework, we observe that \(\chi\) is closed, and thus it is locally the gradient of a conservation law \(h\). This conservation law is easy to find by inspection, and is given by
    \[h(w_1, w_2) = -w_1^2 + w_2\]

An alternative framework to connect symmetries and conservation laws by \cite{tanaka2021noetherslearningdynamics} considers approximating gradient flow by other dynamics that can be induced from The Bregman Lagrangian, 
introduced by \cite{wibisono2016avariationalperspective}. 
The authors found that not all symmetries of the loss function are also symmetries of the Bregman Lagrangian, a phenomenon they term as \textit{symmetry breaking}. 
Nevertheless, by continuing with the derivations, they found that the usual Noether charge satisfies an evolution equation, which they term the \textit{Noether's Learning Dynamics}. 
By taking the mass-less limit of this equation, they recover an equation equivalent to the infinitesimal symmetry condition \ref{eqn:kunin_condition}.  
Their limiting identity recovers gradient-flow conservation properties, but does not provide a general criterion for a prescribed generator to admit a conserved potential.

Finally, \citet{gluch2021noetherthingschangestay} study symmetries and conservation laws of dynamics of the form
\[
\kappa_2(t) \ddot \theta + \kappa_1(t) \dot \theta + \nabla L(\theta) = 0,
\] 
induced by a Lagrangian similar to \cite{tanaka2021noetherslearningdynamics}. 
Although the authors also discuss an interpretation of gradient flow as a massless limit of this dynamic, their derivation does not require any construction of a Lagrangian. 
To derive conservation laws from symmetries, they again start from the infinitesimal symmetry condition \ref{eqn:kunin_condition}.
The main observation is that if the term in the infinitesimal symmetry condition \(\langle\dot \theta, D_g(0, \theta)\rangle\) is exactly the time derivative of some potential function, then that function is a conserved quantity along the dynamic. 
This approach generalizes those presented by \cite{kunin2021neuralmechanicssymmetrybroken, zhao_symmetries_2023}, in the sense that it does not prescribe any particular form of the conserved quantity. 
Rather, in cases where they can find a potential function for the infinitesimal symmetry condition \ref{eqn:kunin_condition}, they deduce that such function is a conserved quantity, while in cases where they cannot find such a potential function, they retain the identity as a constraint. 

For the purpose of finding conservation laws of gradient flow, our framework refines the conservation law construction of \cite{gluch2021noetherthingschangestay}. 
In particular, if \(h(\theta)\) is a conservation law then 
\[
\frac{d}{dt} h(\theta) = \langle \dot \theta, \nabla h(\theta) \rangle = 0.
\]
Thus, if we define a partial symmetry \(\psi\) with infinitesimal generator \(\chi = \nabla h\), then the infinitesimal symmetry condition \ref{eqn:kunin_condition} is satisfied, and \(h\) is a potential function for the infinitesimal symmetry condition.
Through Proposition \ref{prop:conservation_law_vector_field} and  Proposition \ref{prop:symmetry_vector_field}, our framework further makes explicit the integrability criterion underlying their construction.
That is, the infinitesimal symmetry condition 
\[\langle \dot \theta, D_g \psi(0, \theta) \rangle\]
is integrable locally if and only if \(D_g \psi(0, \theta)\) is a closed vector field.

As such, our framework characterizes the fundamental relation underlying previous frameworks by \cite{kunin2021neuralmechanicssymmetrybroken, gluch2021noetherthingschangestay, zhao_symmetries_2023}.
Any sufficiently smooth conservation law of gradient flow derived by these frameworks could be recovered by ours as well.

\subsection{Recent works connecting conservation laws and symmetries via vector fields}

The observation that invariant transformations of the loss function and conservation laws of the gradient flow are related through their infinitesimal generators has been made in several recent works.

\cite{marcotte2025transformativeconservativeconservationlaws} briefly discuss this in their Appendix J.
They define a conservation law induced by a loss invariant transformation as a function \(h\) that has the same gradient as the infinitesimal generator of the transformation itself.

In the case where the symmetry is a linear group action of some Lie group \(G \leq GL(n, \R)\), Corollary~1 of \cite{josz2025subdifferentiationsymmetry} defines a candidate conservation law for gradient flow with non-smooth objective functions of the form
\[h(\theta) = \operatorname{proj}_{s(\mathfrak g)}(\theta \theta^\top),\]
where \(s(\mathfrak g)\) is the symmetric part of the Lie algebra \(\mathfrak g = T_{I_n}(G)\) of \(G\).
The main ingredient in the proof is the observation that such a candidate annihilates the gradient flow vector field.

\cite{gegenfurtner2026symmetries} observe in their Appendix F that if the infitesimal generator of a particular family of symmetries is exact, then its potential function must be conserved. 
They also exhibit an infinitesimal generator of a symmetry that is not closed, and conclude that it does not give rise to a conservation law through this construction. 

Finally, \cite{white2025beyondlagrangians} states similar observations under a variational framework for gradient flow PDEs arising in physics. 

In contrast to prior works, our framework systematically characterizes the precise relationship between local one-parameter loss symmetries and conservation laws through their infinitesimal generators. Here, closedness determine the existence of local conserved potentials, while each sufficiently smooth conservation law determines a unique partial symmetry generated by its gradient, as discussed in Section \ref{sec:general}. 

\section{Justification for 1D Symmetries}
\label{sec:1d_justification}

We motivates the study of \(1\) dimensional symmetries by showing how multidimensional symmetries can be recovered by understanding symmetries induced by the trajectories of the Lie algebra \(\mathfrak g\).

Let \(\Gamma\) be a connected Lie group of dimension \(s\) and \(\psi: \Gamma \times \R^D \to \R^D\) be a smooth symmetry of the loss function \(L\). Let \(\mathfrak g = \operatorname{span}(X_1, \ldots, X_{s})\) be its Lie algebra. For each \(i\), the exponential trajectories \(\{\exp(t X_i) : t \in \R\}\) form a 1-parameter Lie subgroup. Thus, we can define \(s\) 1-dimensional symmetries of \(L\) by taking 
\begin{align*}
    \varphi_i : \R \times \R^D &\to \R^D\\ 
    (t, \theta) &\mapsto \psi(\exp(t X_i), \theta) = \exp(t X_i) \cdot \theta.
\end{align*}
On the other hand, if we know all the 1-parameter Lie subgroups \(\varphi_i\) for \(i = 1, \ldots, d\), then we can recover the multidimensional symmetry \(\psi\) if the Lie group \(\Gamma\) is connected, as can be seen from the following theorem.
\begin{theorem}
    If \(\Gamma\) is a connected Lie group and \(\mathfrak g = \operatorname{span}(X_1, \ldots, X_{s})\) then the product of exponentials of the form
    \(\exp(t_k X_{i_k}) \ldots \exp(t_1 X_{i_1})\)
    generates \(\Gamma\).
\end{theorem}
For a standard reference, see Exercise 8.1 and Proposition 8.3 from \cite{fulton2004representation}.

Thus, a corollary of these facts is that if the 1-parameter Lie subgroups \(\varphi_i(t, \theta) = \exp(t X_i) \cdot \theta\) are known for all \(i\), then we can reconstruct \(\psi: \Gamma \times \R^D \to \R^D\). 
This justifies our focus on characterizing 1-dimensional symmetries for the time being.

\section{Proofs for Section \ref{sec:prelim}}

\subsection{Proof for Proposition \ref{prop:existence_gradient_flow}}
\label{app:proof_existence_gradient_flow}
\existenceGradientFlow*

\begin{proof}
Let \(\mathcal D\) be any finite nonempty dataset and \(\theta_0\) be any initialization. Since these are now fixed, we suppress the dataset \(\mathcal D\) in the notation and write \(L(\theta) := L_{\mathcal D}(\theta)\).
Since summations of continuously differentiable functions are continuously differentiable, \(L\) is continuously differentiable.
Similarly, since summations of locally Lipschitz functions are locally Lipschitz, \(\nabla_\theta L\) is locally Lipschitz.
By the Picard--Lindel\"of theorem, for every
\(\theta_0\in\mathbb R^D\), there exists a unique local solution
\[
\theta:[0,T)\longrightarrow\mathbb R^D
\]
of
\[
\dot\theta(t)=-\nabla_\theta L(\theta(t)),
\qquad
\theta(0)=\theta_0,
\]
for some \(T\in(0,\infty]\).
By a standard continuation argument, there exists a maximal \(T_{\max}\in(0,\infty]\) such that the solution can be extended to \([0,T_{\max})\).
It remains to prove that \(T_{\max}=\infty\). 

Note that \(T_{\max} < \infty\) if and only if there is a finite time blow-up of the solution, i.e.,
\[\lim_{t \to T_{\max}^-} \|\theta(t)\| = \infty.\]
 However, along the solution, the
chain rule gives the identity
\begin{align*}
\frac{\mathrm d}{\mathrm dt}L(\theta(t))
=
\left\langle
\nabla_\theta L(\theta(t)),\dot\theta(t)
\right\rangle 
=
-\left\|\nabla_\theta L(\theta(t))\right\|^2 
=
-\|\dot\theta(t)\|^2.
\end{align*}
Integrating from \(0\) to \(t<T_{\max}\), and noting that the loss function \(L\) is bounded from below since \(\ell\)  itself is bounded from below, we obtain
\[
\int_0^t\|\dot\theta(s)\|^2\,\mathrm ds
=
L(\theta_0)-L(\theta(t)) \leq L(\theta_0) - \inf_{\theta \in \mathbb R^D} L(\theta).
\]
By the Cauchy--Schwarz inequality,
\begin{align*}
\|\theta(t)-\theta_0\|
= 
\left\|\int_0^t\dot\theta(s)\,\mathrm ds\right\|
\leq
\int_0^t\|\dot\theta(s)\|\,\mathrm ds
\leq
\sqrt{t}
\left(
\int_0^t\|\dot\theta(s)\|^2\,\mathrm ds
\right)^{1/2}
&\leq
\sqrt{t(L(\theta_0) - \inf_{\theta \in \mathbb R^D} L(\theta))}.
\end{align*}
Thus
\[
\|\theta(t)\| \leq \|\theta_0\| + \sqrt{t (L(\theta_0) - \inf_{\theta \in \mathbb R^D} L(\theta))},
\]
and hence no finite time blow-up is possible. 

We conclude that for every finite nonempty dataset and every
\(\theta_0\in\mathbb R^D\), the gradient-flow equation admits a unique
solution defined for all \(t\in[0,\infty)\).
\end{proof}

\subsection{Proof for Proposition \ref{prop:independentfunctions}}
\label{app:proof_independent_functions}

\independentFunctions*

\begin{proof}
    The backward direction is obtained by taking the derivative of Equation \ref{eq:functional_dependence}. Indeed, for any \(\theta \in \Theta\), the associated \(f\) induces the required linear coefficients since
    \[
        \nabla_\theta h(\theta) = \sum_{i=1}^k \frac{\partial f}{\partial x_i} (H(\theta)) \nabla_\theta h_i (\theta) \in \operatorname{span}\{\nabla_\theta h_1(\theta), \ldots, \nabla_\theta h_k(\theta)\}.
    \]

    Now, assume (i) holds. Since the gradients of the \(h_i\) are linearly independent in \(\Theta\), \(H\) is a submersion from \(\Theta \subset \R^D\) to \(\R^k\). Therefore, by the submersion theorem \cite[Theorem 11.5]{tu2010introduction}, there exists an open neighborhood \(\theta_0 \in V \subset \Theta\) and diffeomorphisms \(\phi: V \to \R^D, \psi: H(V) \to \R^k\) such that
    \[\psi \circ H \circ \phi^{-1}: \R^D \to \R^k\] is the orthogonal projection \(\operatorname{proj}: \R^D \to \R^k\) of the first \(k\) coordinates, i.e., 
    \[\psi \circ H \circ \phi^{-1}(y, z) = \operatorname{proj}(y,z) = y, \quad \forall y \in \R^k, z \in \R^{D-k}.\]
    
    Intuitively, \(\phi\) and \(\psi\) flatten both the domain and the range of \(H\). 
    This makes explicit the idea that up to diffeomorphisms, \(H\) can only see variations in \(k\) directions. 
    Next, by assumption (i), we will show that these directions entirely encompass the direction of variation of \(h\) up to diffeomorphism.

    We claim that \(h \circ \phi^{-1}: \R^D \to \R\) is independent of the last \(D-k\) coordinates, and hence  \(\overline h: \R^k \to \R\) defined by 
    \[\overline h(y) = h \circ \phi^{-1} (y,z), \quad \forall z \in \R^{D-k}\]
    is a well-defined function.
    If this is the case, then the projection map \(\operatorname{proj}: \R^D \to \R^k\) to the first \(k\) coordinates satisfies
    \[h \circ \phi^{-1}(y, z) = \overline h(y) = \overline h(\operatorname{proj}(y,z)),\]
    and we would have for any \(\theta \in V\)
    \begin{align*}
        h(\theta) &= h \circ \phi^{-1}(\phi(\theta)) = \overline h(\operatorname{proj}(\phi(\theta)) = \overline h \circ \psi \circ H \circ \phi^{-1} \circ \phi (\theta) = \overline h \circ \psi (H(\theta)) = \overline h \circ \psi (h_1(\theta), \ldots, h_k(\theta)).
    \end{align*}
    Thus we can choose \(f(x) := \overline h \circ \psi(x) = h \circ \phi^{-1} (\psi(x), 0)\). The fact that \(f\) is \(C^r\) is inherited from \(h\). An illustrative commutative diagram is included in Figure \ref{fig:functional_dependence} for ease of following the proof.
    \begin{figure}
        \centering
        \includegraphics[width=0.7\linewidth]{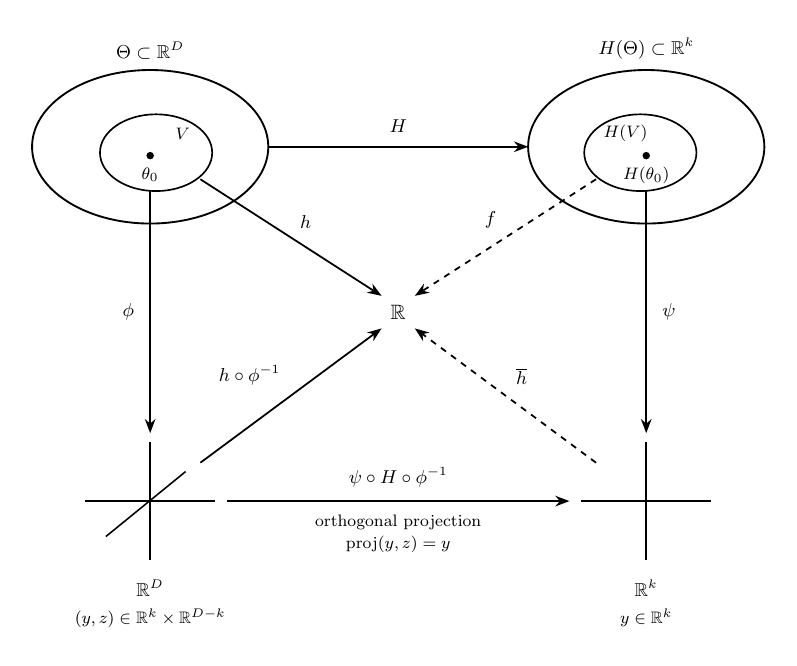}
        \caption{Commutative diagram showing the domains and ranges of different functions in the proof. The key step is to construct \(\overline h\). The final construction is \(h = f \circ H = \overline h \circ \psi \circ H\).}
        \label{fig:functional_dependence}
    \end{figure}

    The final step is to prove the claim that \(h \circ \phi^{-1}: \R^D \to \R\) is independent of the last \(D-k\) coordinates.
    Let \(y \in \R^k, z \in \R^{D-k}\).
    By chain rule and the fact that \(\psi \circ H \circ \phi^{-1}\) is the orthogonal projection, observe that
    \[0 = D_z \evalat{\left(\psi \circ H \circ \phi^{-1}\right)}{(y,z)} = \evalat{D \psi}{H \circ \phi^{-1}(y,z)} \, \evalat{DH}{\phi^{-1}(y,z)} \, \evalat{D_z \phi^{-1}}{(y,z)}.\]
    Since \(\psi\) is a diffeomorphism, this implies \(\evalat{DH}{\phi^{-1}(y,z)} \, \evalat{D_z \phi^{-1}}{(y,z)} = 0\), and thus
    \[\evalat{D h_i}{\phi^{-1}(y,z)} \, \evalat{D_z \phi^{-1}}{(y,z)} = 0, \quad \forall i \in \overline{1, \ldots, k}.\]
    By the chain rule and assumption (i), we obtain
    \[\evalat{D_z (h \circ \phi^{-1})}{(y,z)} = \evalat{D h}{\phi^{-1}(y,z)} \, \evalat{D_z \phi^{-1}}{(y,z)} = 0.\]
    Thus, \(h \circ \phi^{-1}\) is independent of \(z\), completing the proof.
\end{proof}

\section{Proof for Section \ref{sec:general}}

\subsection{Proof for Proposition \ref{prop:symmetry_equivalence}}

\symmetryEquivalence*

\begin{proof}
Suppose \(L_{\mathcal D}\) has the differentiable symmetry \(\Gamma\), then 
\[\mathcal L(g, \theta) := L_{\mathcal D} \circ \psi (g, \theta) = L_{\mathcal D}(\psi(g, \theta))\] is a constant function with respect to \(g\). Thus, for all \(g \in \Gamma, \theta \in \Theta\), we have
\[0 = \evalat{(D_g \mathcal L)}{(g, \theta)} = \evalat{D L_{\mathcal D}}{\psi(g, \theta)} \evalat{D_g \psi}{(g, \theta)}.\]
In particular, for \(g = e\), we have
\[\evalat{(D_g \mathcal L)}{(e, \theta)} = \evalat{D L_{\mathcal D}}{\psi(e, \theta)} \evalat{D_g \psi}{(e, \theta)} = \evalat{D L_{\mathcal D}}{\theta} \evalat{D_g \psi}{(e, \theta)} = D L_{\mathcal D}(\theta) \; D_g \psi(e, \theta) = 0, \quad \forall \theta \in \Theta.\]
Taking the transpose yields \eqref{eq:symmetry_condition}.

Conversely, suppose \eqref{eq:symmetry_condition} holds for all \(\theta \in \Theta\). 
For any \(g \in \Gamma\) and any \(v \in T_g \Gamma\), there exists a smooth path \(\gamma: [0,1] \to \Gamma\) such that \(\gamma(0) = g, \dot \gamma(0) = v\). 
We can reparametrize 
\[\gamma(t) =(\gamma(t) g^{-1}) g =: \epsilon(t) g, \quad \epsilon(0) = e.\]
Then
\begin{align*}
    \frac{d}{dt}\mathcal L(\gamma(t), \theta) 
    = \frac{d}{dt} L_{\mathcal D}(\psi(\epsilon(t) g, \theta))
    = \frac{d}{dt} L_{\mathcal D}(\psi(\epsilon(t),\psi(g, \theta)))
    = \evalat{DL_{\mathcal D}}{\psi(\epsilon(t),\psi(g, \theta))} \evalat{D_g \psi}{(\epsilon(t),\psi(g, \theta))} \dot \epsilon(t).
\end{align*}
In particular, for \(t = 0\), we have
\[
\evalat{\frac{d}{dt} \mathcal L(\gamma(t), \theta)}{t = 0} 
= \evalat{DL_{\mathcal D}}{\psi(e,\psi(g, \theta))} \evalat{D_g \psi}{(e,\psi(g, \theta))} \dot \epsilon(0) 
= \evalat{DL_{\mathcal D}}{\psi(g, \theta)} \evalat{D_g \psi}{(e,\psi(g, \theta))} \dot \epsilon(0) = 0,
\]
where the last equality follows from the assumption \eqref{eq:symmetry_condition}. 
This proves that 
\[D_g \mathcal L(g, \theta)[v] = 0, \forall v \in T_g \Gamma.\]
Since \(v \in T_g \Gamma\) was arbitrary, we have shown that \(D_g \mathcal L(g, \theta) = 0\) for all \(g \in \Gamma\).

Finally, by the connectedness of \(\Gamma\), there exists a smooth path connecting any \(g \in \Gamma\) to the identity \(e\), and since the derivative along any such path is zero, we conclude that 
\[\mathcal L(g, \theta) = \mathcal L(e, \theta) = L_{\mathcal D}(\psi(e, \theta)) = L_{\mathcal D}(\theta)\] for all \(g \in \Gamma\).
\end{proof}

\subsection{Proof for Corollary \ref{cor:1dsymmetrycond}}

\oneDimensionalSymmetryCondition*

\begin{proof}
    Replace \(\Gamma\) with the additive group \(\R\) and the identity element \(e\) with \(0\) in Proposition \ref{prop:symmetry_equivalence}.
    Since \(D_t \psi(0, \theta)\) is a vector field, we can rewrite the equation \ref{eq:symmetry_equivalence_one_parameter} using the inner product notation as desired.
\end{proof}

\subsection{Proof for Corollary \ref{cor:1dsymmetrycond2}}
\globalOneDimensionalSymmetryCondition*

\begin{proof}
    The forward implication follows directly from Corollary \ref{cor:1dsymmetrycond} by noting that \(W_\theta\) is the span of the gradient flow directions for all datasets of size 1.

    For the reverse implication, for any dataset \(\mathcal D\), its gradient flow vector field 
    \[- \nabla_\theta L_{\mathcal D}(\theta) = - \frac 1 n \sum_{i=1}^n \nabla_\theta \ell_{(x^{(i)}, y^{(i)})}(\theta)\]
    belongs to \(W_\theta\). 
    Therefore, if \eqref{eq:symmetry_equivalence_one_parameter2} holds then 
    \[\langle D_t \psi(0, \theta), - \nabla_\theta L_{\mathcal D}(\theta) \rangle = 0, \quad \forall \theta \in \Theta.\]
    Corollary \ref{cor:1dsymmetrycond} then implies that \(\psi\) is a symmetry of \(L_{\mathcal D}\). 
    Since the dataset \(\mathcal D\) was arbitrary, this completes the proof.
\end{proof}

\subsection{Proof for Proposition \ref{prop:conservation_law_vector_field}}

\conservationLawVectorField*

\begin{proof}
    The first part of the proposition follows directly from the definition of a conservation law and the fact that the gradient of any scalar function is a closed vector field.
    The second part follows from the Poincar\'e lemma \cite[Theorem 11.50]{lee2003introduction}.
    The final part is the Poincar\'e lemma itself \cite[Theorem 11.49]{lee2003introduction}.
\end{proof}

\subsection{Proof for Proposition \ref{prop:symmetry_vector_field}}

\symmetryVectorField*
\begin{proof}
    The first part of the proposition follows directly from Corollary \ref{cor:1dsymmetrycond2}.

    The second part of the proposition follows from the standard theory of ordinary differential equations. 
    Since \(\chi\) is smooth, the existence and uniqueness of the maximal local flow \(\psi\) is guaranteed by the Picard-Lindel\"off theorem. 
    The first three properties of the maximal local flow listed in the proposition are standard results for the flow of a smooth vector field.
    The last property, \(L(\psi(t, \theta)) = L(\theta)\), follows from the fact that the flow direction \(\chi \in W^\perp\) is always perpendicular to the gradient flow distribution \(W\).

    Moreover, if the maximal local flow exists for all time, then by Corollary \ref{cor:1dsymmetrycond2}, \(\psi\) is a symmetry of the loss function \(L\).
\end{proof}

\subsection{Extension for Corollary \ref{cor:1dsymmetrycond}, Corollary \ref{cor:1dsymmetrycond2} and Propostion \ref{prop:symmetry_vector_field} to partial symmetries}
\label{app:partial_symmetry_extension}

It would have been more convenient to state all results in the context of partial symmetries, but we chose to first present the results in the context of global symmetries for clarity.

Every global symmetry is of course also a partial symmetry.
The nuance is just that you might not always be able to integrate a vector field to a global symmetry
Nevertheless, you can always integrate it locally to obtain a partial symmetry, and in a small neighborhood of a point, this partial symmetry is indistinguishable from a global symmetry.
Alternatively, if we assume that all vector fields that concern us are complete, then we can temporarily forget about the distinction between partial and global symmetries.

\begin{proposition}[Extension of Corollary \ref{cor:1dsymmetrycond}]
    \label{prop:extension_cor_1d}
    Let \(\Gamma = \R\) be the additive group of real numbers that acts on \(\Theta\) via a partial group action \(\psi(t, \theta)\). 
    Let \(\mathcal D\) be a dataset in \((\mathcal X \times \mathcal Y)^\ast\).
    A necessary and sufficient condition for \(\Gamma\) to be a differentiable symmetry of \(L_{\mathcal D}\) is
    \begin{equation*}
        \langle D_t \psi(0, \theta) ,\nabla_\theta L_{\mathcal D}(\theta)\rangle = 0, \quad \forall \theta \in \Theta.
    \end{equation*}
\end{proposition}
\begin{proof}
    \((\Gamma, \psi)\) is a partial symmetry of \(L_{\mathcal D}\) if and only if \(L_{\mathcal D}(\psi(t, \theta))\) is constant with respect to \(t\) whenever defined, if and only if \(\frac{d}{dt} L_{\mathcal D}(\psi(t, \theta)) = 0\) whenever defined.
    By the chain rule, this is equivalent to 
    \[\langle D_t \psi(t, \theta) ,\nabla_\theta L_{\mathcal D}(\psi(t, \theta))\rangle = 0\] 
    whenever defined.
    Setting \(t = 0\) yields the forward implication.

    On the other hand, suppose \ref*{eq:symmetry_equivalence_one_parameter} holds for all \(\theta \in \Theta\).
    Then it must hold that 
    \[\langle D_t \psi(0, \psi(s,\theta)), \nabla_{\theta} L_D(\psi(s,\theta)) \rangle = 0\]
    for all \(s\) whenever defined.
    But since \(\psi\) satisfies the group action property whenever defined, we get
    \[D_t \psi(0, \psi(s, \theta)) = D_t \psi(s, \theta).\]
    Thus, \(\langle D_t \psi(s, \theta), \nabla_\theta L_{\mathcal D}(\psi(s,\theta))\rangle = 0\) for all \(s\) whenever defined, which is the equivalent condition above.
\end{proof}

\begin{proposition}[Extension of Corollary \ref{cor:1dsymmetrycond2}]
    \label{prop:extension_cor_1d2}
    Let \(\Gamma = \R\) be the additive group of real numbers that acts on \(\Theta\) via a partial group action \(\psi(t, \theta)\). 
    A necessary and sufficient condition for \(\Gamma\) to be a partial symmetry of \(L\) is
    \begin{equation*}
        D_t \psi(0, \theta) \perp W_\theta, \quad \forall \theta \in \Theta.
    \end{equation*}
\end{proposition}
\begin{proof}
    The proof is identical to that of Corollary \ref{cor:1dsymmetrycond2}, except that we rely on Proposition \ref{prop:extension_cor_1d}, i.e., the extension of Corollary \ref{cor:1dsymmetrycond} to partial symmetries, instead of Corollary \ref{cor:1dsymmetrycond} itself.
\end{proof}

\begin{proposition}[Extension of Proposition \ref{prop:symmetry_vector_field}]
    \label{prop:extension_symmetry_vector_field}
    If \((\R, \psi)\) is a partial symmetry of the loss function \(L\), then \(\chi := D_t \psi \mid_{t = 0}\) is a smooth vector field in \(W^\perp\).
    On the other hand, if \(\chi \in W^\perp\) is a smooth vector field, then there exists a unique maximal local flow \(\psi: \mathcal O \to \Theta\), where \(\mathcal O \subset \R \times \Theta\) is the maximal open set containing \(\{0\} \times \Theta\), for which the solution of 
    \[\frac{d}{dt} \psi(t, \theta) = \chi(\psi(t, \theta)), \quad \psi(0, \theta) = \theta\]
    is well-defined. 
    This unique maximal local flow satisfies the following four equations for all \(\theta \in \Theta\) and \(t,s\) whenever the terms are well-defined:
    \begin{align*}
        &\psi(0, \theta) = \theta, \qquad 
        \psi(t + s, \theta) = \psi(t, \psi(s, \theta)), &\text{(group action axioms)}\\
        &D_t \psi(0, \theta) = \chi(\theta), &\text{(infinitesimal generator)}\\
        &L(\psi(t, \theta)) = L(\theta). &\text{(loss invariance condition)}
    \end{align*}
    Furthermore, if \(\chi\) is complete, then \(\psi\) is a symmetry of the loss function \(L\).
\end{proposition}

\begin{proof}
    The first part of the proposition follows directly from Proposition \ref{prop:extension_cor_1d2}.
    The second part contains no difference compared to the corresponding part in Proposition \ref{prop:symmetry_vector_field}, and hence does not require a separate proof.
\end{proof}

\subsection{Proof for Corollary \ref{cor:law_induces_symmetry}}
\lawInducesSymmetry*
\begin{proof}
    Let \(h\) be a conservation law, then Proposition \ref{prop:conservation_law_vector_field} implies that \(\chi := \nabla h\) is a vector field in \(W^\perp\).
    Applying Proposition \ref{prop:symmetry_vector_field} to the vector field \(\chi\) completes the proof.

    On the other hand, let \(\psi\) be a partial symmetry of \(L\), then by Proposition \ref{prop:symmetry_vector_field}, its infinitesimal generator \(\chi := D_t \psi(0, \theta)\) is a vector field in \(W^\perp\).
    If \(\chi\) is also closed, then by Proposition \ref{prop:conservation_law_vector_field}, there exists a local conservation law \(h\) such that \(\nabla h = \chi\), up to an additive constant.
\end{proof}

\subsection{Proof for Theorem \ref{thm:law-symmetry-gap}}

\lawSymmetryGap*

\begin{proof}
    Since \(W\) is a smooth distribution of constant rank \(r\), its orthogonal complement \(W^\perp\) is a smooth distribution of constant rank \(D - r\).
    Thus, there exists a local smooth frame \(\chi_1, \ldots, \chi_{D-r}\) that completely spans \(W^\perp\) in a neighborhood \(U \ni \theta\).
    By Proposition \ref{prop:symmetry_vector_field}, each \(\chi_i\) generates a unique partial symmetry \((\R, \psi_i)\) of the loss function \(L\) in \(U\).
    Any other partial symmetry \((\R, \psi)\) of \(L\) must have its infinitesimal generator \(D_t \psi(0, \theta)\) in \(W^\perp\), and hence it is infinitesimally spanned by \(\chi_1, \ldots, \chi_{D-r}\) in \(U\).
    This proves (a).

    Since \(\operatorname{Lie}(W)\) is a smooth, involutive distribution of constant rank \(\overline r\), its orthogonal complement \(\operatorname{Lie}(W)^\perp\) is a smooth distribution of constant rank \(D - \overline r\). 
    By the Frobenius theorem, there exists a local coordinate system
    \[u_1, \ldots, u_{\overline r}, v_1, \ldots, v_{D - \overline r}\]
    around any \(\theta \in \Theta\) such that the vector fields \(\nabla v_1, \ldots, \nabla v_{D - \overline r}\) completely span \(\operatorname{Lie}(W)^\perp\)  in a neighborhood \(U \ni \theta\).
    Note that \(\nabla v_i\) is an exact vector field in \(W^\perp\).
    By Proposition \ref{prop:conservation_law_vector_field}, the functions \(h_i := v_i\) for \(i = 1, \ldots, D - \overline r\) are conservation laws of the loss function \(L\) in \(U\).
    Any other conservation law \(h\) must have its gradient \(\nabla h\) lie in \(\operatorname{Lie}(W)^\perp\), and hence it is infinitesimally spanned by \(\nabla h_1, \ldots, \nabla h_{D - \overline r}\) in \(U\).
    This proves (b).

    Finally, since \(W \subseteq \operatorname{Lie}(W)\), we have \(r \leq \overline r\), and hence \(s - c = \overline r - r \geq 0\).
    Equality holds if and only if \(r = \overline r\), which is equivalent to \(W\) being involutive, i.e., \(\operatorname{Lie}(W) = W\).
\end{proof}

\section{Proof for Section \ref{sec:multilayer}}

\subsection{Loss functions that satisfy Assumption \ref{as:separability}}
\label{app:loss_assumption}

\separability*

Here, we review some common loss functions that satisfy Assumption \ref{as:separability}.
This is not an exhaustive list since the assumption is very general.
The idea is that the loss function should be able to distinguish between different outputs, which is a very mild requirement for most loss functions used in practice.
We do not restrict the list to only loss functions with a specific regularity that satisfies our smoothness assumption because this is a technical assumption that may be relaxed in future work.
    
\begin{restatable}{lemma}{lossSeparability}
    \label{lem:loss-separability}
    The following loss functions satisfy Assumption \ref{as:separability}:
    \begin{enumerate}
        \item Assuming \(\mathcal Z = \mathcal Y\), any loss function \(\ell\) that satisfies \(\ell(z, z)= 0\) and \(\ell(z, y) > 0\) for all \(z \neq y\).
        \item Assuming \(\mathcal Z = \mathcal Y\), any loss function \(\ell\) that is also a metric on \(\mathcal Z\), such as the Euclidean distance, or any \(L^p\) distance.
        \item The quadratic loss 
            \[\ell:\R^{n} \times \R^n \to \R, \quad \ell(z, y) = \|z - y\|^2_2.\]
        \item The absolute loss
            \[\ell:\R^{n} \times \R^n \to \R, \quad \ell(z, y) = \|z - y\|_1.\]
        \item Any \(L^p\) loss for \(p \in (0, \infty)\)
            \[\ell:\R^{n} \times \R^n \to \R, \quad \ell(z, y) = \|z - y\|_p^p.\]
        \item The multiclass cross-entropy loss on the probability simplices
            \[\ell: \Delta_{++}^{V-1} \times \Delta_{+}^{V-1} \to \R, \quad \ell(z, y) = -\sum_{i = 1}^V y_i \log z_i.\]
        \item The KL divergence loss on the probability simplices
            \[\ell: \Delta_{++}^{V-1} \times \Delta_{+}^{V-1} \to \R, \quad \ell(z, y) = \sum_{i = 1}^V y_i \log\frac{y_i}{z_i}.\]
        \item The cosine loss on unit vectors
            \[\ell: S^{n-1} \times S^{n-1} \to \R, \quad \ell(z, y) = 1 - z^\top y.\]
    \end{enumerate}
\end{restatable}

\begin{proof}
\mbox{} 
\begin{enumerate}
    \item Choose \(y = z\) to obtain \(\ell(z, z) = 0\), then \(\ell(z', z) = l(z, z) = 0\), implying \(z' = z\) by the assumption that \(\ell(z, y) > 0\) for all \(z \neq y\).
    \item Since \(\ell\) is a metric, it satisfies the condition 1 above.
    \item The quadratic loss satisfies the condition 1 above.
    \item The absolute loss is a metric.
    \item The \(L^p\) loss is a metric for \(p \geq 1\). For \(p \in (0, 1)\), the \(L^p\) loss is not a metric, but it still satisfies the condition 1 above.
    \item  Multiclass cross-entropy loss: assume that
    \[
    -\sum_{i=1}^V y_i\log z_i
    =
    -\sum_{i=1}^V y_i\log z'_i
    \qquad
    \text{for all }y\in\Delta_{+}^{V-1}.
    \]

    For each \(k\in\{1,\ldots,V\}\), choose \(y=e_k\), where \(e_k\) is the
    \(k\)-th standard basis vector. Since \(e_k\in\Delta_{+}^{V-1}\), we obtain
    \[
    -\log z_k=-\log z'_k.
    \]
    Because the logarithm is injective, we must have \(z_k=z'_k\) for every index \(k\). Since this holds for every \(k\), we get \(z=z'\).
    \item The KL divergence loss satisfies the condition 1 above. Indeed, applying the inequality \(\log t \leq t -1, \forall t > 0\), we get
    \begin{align*}
        D_{\mathrm{KL}}(y\|z) = \sum_{i} y_i \log\frac{y_i}{z_i} = - \sum_i y_i \log\frac{z_i}{y_i} \geq - \sum_i y_i \left(\frac{z_i}{y_i} - 1\right) = - \sum_i (z_i - y_i) = 0,
    \end{align*}
    with equality if and only if \(z_i = y_i\) for all \(i\).
    \item The cosine loss satisfies the condition 1 above. Indeed, since \(z\) and \(y\) are unit vectors, the only way to have \(\ell(z, y) = 0\) is if \(z = y\).
    \end{enumerate}
\end{proof}

\subsection{Proof for Proposition \ref{prop:separability-symmetry}}

\separabilitySymmetry*

\begin{proof}
    Let \(\psi\) be a partial symmetry of the loss function \(L\). 
    This means that for all \((x, y) \in \mathcal X \times \mathcal Y\),
    \[\ell(G(\psi(t, \theta), x), y) = \ell(G(\theta, x), y).\]
    By Assumption \ref{as:separability}, this implies that for all \(x \in \mathcal X\),
    \[G(\psi(t, \theta), x) = G(\theta, x).\]
    This shows that \(\psi\) is a functional symmetry of the model \(G\).
\end{proof}

\subsection{Proof for Proposition \ref{prop:inheritance_1}}
\label{subsec:proof-inheritance-1}

Before proceeding to the proof, we first provide Table~\ref{tab:inheritance_examples} of examples of function combinations that satisfy the identifiability condition (\ref{eqn:indentifiability_1}) in Proposition \ref{prop:inheritance_1}.

\begin{table}[ht!]
    \caption{Examples of function combinations compatible with
Proposition~\ref{prop:inheritance_1}.
The listed operations occur within the indicated architectures. The categories are not mutually exclusive. The list is not exhaustive. We write $g_\theta(u)=g(\theta,u)$, and $\mathcal N(i)$ is the neighborhood of node $i$.} 
\label{tab:inheritance_examples}
\begin{center}
\small
\setlength{\tabcolsep}{5pt}
\renewcommand{\arraystretch}{1.15}

\begin{tabularx}{\textwidth}{
    @{}
    >{\raggedright\arraybackslash}p{0.25\textwidth}
    >{\raggedright\arraybackslash}p{0.35\textwidth}
    >{\raggedright\arraybackslash}X
    @{}
}
\textbf{Combination Type}
    & \textbf{Operation}
    & \textbf{Typical Examples} \\
\midrule

Additive
    & $z = g_\theta(x) + f_\omega(x)$
    & ResNet \citep{He2015DeepRL}, LoRA \citep{hu2022lora} \\
\addlinespace

Concatenation
    & $z = [g_\theta(x);\, f_\omega(x, g_\theta(x))]$
    & Inception \citep{Szegedy2015}, U-Net \citep{Ronneberger2015unet} \\
\addlinespace

Composition
    & $
        z = f_\omega(g_\theta(x)), \text{ or }
        z = g_\theta(f_\omega(x))$
    & MLPs \citep{Rumelhart1986LearningRB}, CNNs \citep{lecun98gradientbased} \\
\addlinespace

Hadamard Product / Gating
    & $z = g_\theta(x) \odot f_\omega(x)$
    & LSTM \citep{hochreiter1997long} and GRU gates \citep{cho-etal-2014-learning} \\
\addlinespace

Input-dependent Weighted Sum
    & $z = a_\omega(x)\,g_\theta(x) + f_\omega(x)$
    & Mixture-of-Experts \citep{shazeer2017outrageously} \\
\addlinespace

Outer Product
    & $Z = g_\theta(x)\,f_\omega(x)^\top$
    & Bilinear CNNs \citep{Lin_2015_ICCV} \\
\addlinespace

Graph Aggregation
    & $z_i =
       \operatorname{AGG}_\omega
       \bigl(x_i, \{g_\theta(x_j, x_i, e_{ij}):j\in\mathcal N(i)\}\bigr)$
    & MPNNs \citep{gilmer2017neuralmessagepassing} \\
\addlinespace

Shared Weights Across Branches
    & $z_i = f_{i,\omega}(x,g_\theta(x)),\quad i=1,\ldots,m$
    & Grouped-query attention \citep{ainslie2023gqa}\\
\addlinespace

Repeated Composition
    & $z_t = g_\theta(z_{t-1},x_t),$
    & RNNs \citep{cho-etal-2014-learning}\\
\end{tabularx}
\end{center}
\end{table}

\inheritanceI*

\begin{proof}
\mbox{} 
    \begin{itemize}[leftmargin=*]
        \item Let \((\R, \psi)\) be a partial loss symmetry of the subnetwork \(g\) with respect to a loss satisfying Assumption~\ref{as:separability}.
            By Proposition \ref{prop:separability-symmetry}, \((\R, \psi)\) is also a partial functional symmetry of \(g\), i.e.,
            \[g(\psi(t, \theta), x) = g(\theta, x), \quad \forall \theta \in \Theta, x \in \mathcal X^g, t \text{ whenever well-defined}.\]
            Condition \ref{eqn:indentifiability_1} implies that for all \((\theta, \omega) \in \Theta \times \Omega\), \(x \in \mathcal X\), and \(t\) whenever well-defined,
            \[
                G(\tilde{\psi} (t, (\theta, \omega)), x) 
                = G((\psi(t, \theta), \omega), x) 
                = G((\theta, \omega), x).
            \]
            Thus, \((\R, \tilde{\psi})\) is a partial functional symmetry of the larger network \(G\), and hence also a partial loss symmetry of \(G\) with respect to any loss function.

        \item Let \(h\) be a conservation law of the subnetwork \(g\) with respect to a loss satisfying Assumption~\ref{as:separability}.
            By Corollary \ref{cor:law_induces_symmetry}, there exists a unique partial loss symmetry \((\R, \psi)\) of \(g\) associated with the underlying vector field \(\chi := \nabla h\).
            The previous argument shows that the extension \((\R, \tilde{\psi})\) is a partial loss symmetry of the larger network \(G\) for any loss function. 
            
            Importantly, the underlying vector field of \((\R, \tilde{\psi})\) is exactly the extension of \(\chi\) to the larger parameter space \(\Theta \times \Omega\), i.e.,
            \[\tilde \chi := D_t \tilde \psi (0, (\theta, \omega)) = (\chi, 0).\]

            By Proposition \ref{prop:symmetry_vector_field}, the underlying vector field \(\tilde \chi\) of \((\R, \tilde{\psi})\) is in the annihilator of the gradient flow distribution of the larger network \(G\), i.e., \(\tilde \chi \in W^\perp\) for any loss function.
            Note that this vector field is exactly the gradient of the extended function \(\tilde h\), i.e., \(\tilde \chi = \nabla \tilde h\).
            Since \(\tilde \chi \in W^\perp\) and \(\tilde h\) is a scalar potential function of \(\tilde \chi\), Proposition \ref{prop:conservation_law_vector_field} implies that \(\tilde h\) is a conservation law of the larger network \(G\) with respect to any loss function.
    \end{itemize}
\end{proof}

\subsection{Proof for Proposition \ref{prop:inheritance_2}}
\inheritanceII*

\begin{proof}
\mbox{} 
    \begin{itemize}[leftmargin=*]
        \item Let \((\R, \psi)\) be a partial loss symmetry of the larger network \(G\) with respect to a loss satisfying Assumption \ref{as:separability}, where for any \(t\) and \((\theta_1, \theta_2) \in \Theta^1 \times \Theta^2\), we have
        \[\psi(t, (\theta_1, \theta_2)) =
        \begin{pmatrix}
            \psi^1(t, (\theta_1, \theta_2))\\
            \psi^2(t, (\theta_1, \theta_2))
        \end{pmatrix} \in \Theta^1 \times \Theta^2 . 
        \]
            By Proposition \ref{prop:separability-symmetry}, \((\R, \psi)\) is also a partial functional symmetry of \(G\), i.e.,
            \[
            G(\psi(t, (\theta_1, \theta_2)), \cdot) = G((\theta_1, \theta_2), \cdot), \quad \forall (\theta_1, \theta_2) \in \Theta^1 \times \Theta^2, t \text{ whenever well-defined}.
            \]
            Condition \ref{eqn:identifiability_2} now implies that for all \(j = 1, 2\), we have
            \begin{equation*}
                g^j(\psi^j(t, (\theta_1, \theta_2)), \cdot) = g^j(\theta_j, \cdot), \quad \forall (\theta_1, \theta_2) \in \Theta^1 \times \Theta^2, t\text{ whenever well-defined}.
            \end{equation*}

            Observe that the right hand side of the previous equation does not depend on \(t\), and hence so does the left hand side. 
            By taking its derivative with respect to \(t\) at \(t = 0\), we obtain
            \[D_{\theta_j} g^j (\psi^j(0, (\theta_1, \theta_2)), \cdot) \, D_t \psi^j(0, (\theta_1, \theta_2)) = 0, \quad \forall (\theta_1, \theta_2) \in \Theta^1 \times \Theta^2.\]
            By definition, the parameter term inside \(g^j\) is 
            \(\psi^j(0, (\theta_1, \theta_2)) = \theta_j\), and so we have
            \begin{equation}
                \label{eqn:individual_symmetry}
                D_{\theta_j} g^j (\theta_j, \cdot) \, D_t \psi^j(0, (\theta_1, \theta_2)) = 0, \quad \forall (\theta_1, \theta_2) \in \Theta^1 \times \Theta^2.
            \end{equation}

            For any fixed \(\theta_2\), \eqref{eqn:individual_symmetry} shows that  the vector field \(D_t \psi^1(0, (\theta_1, \theta_2))\) is an infinitesimal functional symmetry direction of \(g^1\) at any \(\theta_1\).
            Thus, it is infinitesimally generated by the complete set of partial symmetries \((\R, \psi^1_i)_{i=1}^{s_1}\) of \(g^1\), i.e.,
            \[D_t \psi^1(0, (\theta_1, \theta_2)) \in \operatorname{span}\{D_t \psi^1_i(0, \theta_1), i = 1, \ldots, s_1\}.\]
            Note that 
            \[D_t \tilde \psi^1_i(0, (\theta_1, \theta_2)) = \begin{pmatrix}
                D_t \psi^1_i(0, \theta_1)\\
                0
            \end{pmatrix}, \forall i = 1, \ldots, s_1\]
            and so
            \begin{align*}
                \begin{pmatrix}
                D_t \psi^1(0, (\theta_1, \theta_2))\\
                0
            \end{pmatrix} 
            &\in \operatorname{span} \left\{\begin{pmatrix}
                D_t \psi^1_i(0, \theta_1)\\
                0
            \end{pmatrix}, i = 1, \ldots, s_1\right\}
            \\ 
            &= \operatorname{span}\{D_t \tilde \psi^1_i(0, (\theta_1, \theta_2)), i = 1, \ldots, s_1\}.
            \end{align*}

            Similarly, for any fixed \(\theta_1\), the same argument shows that
            \[\begin{pmatrix}
                0\\
                D_t \psi^2(0, (\theta_1, \theta_2))
            \end{pmatrix} \in \operatorname{span}\{D_t \tilde \psi^2_i(0, (\theta_1, \theta_2)), i = 1, \ldots, s_2\}.\]

            Combining the two results shows
            \[D_t \psi(0, (\theta_1, \theta_2)) = \begin{pmatrix}
                D_t \psi^1(0, (\theta_1, \theta_2))\\0
            \end{pmatrix} + \begin{pmatrix}
                0\\ D_t \psi^2(0, (\theta_1, \theta_2))
            \end{pmatrix} \]
            is in the span of \(D_t \tilde \psi^j_i(0, (\theta_1, \theta_2)), j = 1, 2, i = 1, \ldots, s_j\).
            Thus, any partial loss symmetry of the larger network \(G\) is infinitesimally generated by the extensions of the complete set of partial functional symmetries of the subnetworks \(g^1\) and \(g^2\).

            \item Let \(h\) be a conservation law of the larger network \(G\) with respect to a loss satisfying Assumption \ref{as:separability}.
            By Corollary \ref{cor:law_induces_symmetry}, there exists a unique partial loss symmetry \((\R, \psi)\) of \(G\) given by 
            \[\psi(t, (\theta_1, \theta_2)) = \begin{pmatrix}
                \psi^1(t, (\theta_1, \theta_2))\\
                \psi^2(t, (\theta_1, \theta_2))
            \end{pmatrix}\]
            associated with the underlying vector field \(\chi := \nabla h = (\nabla_{\theta_1} h, \nabla_{\theta_2} h)\).

            The previous argument shows that for any fixed \(\theta_2\), 
            \[\nabla_{\theta_1} h (\theta_1, \theta_2) = D_t \psi^1(0, (\theta_1, \theta_2))\]
            is an infinitesimal functional symmetry direction of \(g^1\) at \(\theta_1\).
            Thus, it is also an infinitesimal loss symmetry direction of \(g^1\), i.e., \(\nabla_{\theta_1} h \in (W^1)^\perp\).
            Note that \(\nabla_{\theta_1} h\) is exactly the gradient of the function \(h(\theta_1, \theta_2)\) with fixed \(\theta_2\).
            Proposition \ref{prop:conservation_law_vector_field} implies that \(h(\cdot, \theta_2)\) is a conservation law of the subnetwork \(g^1\).
            Since \(h^1_1, \ldots, h^1_{c_1}\) is a complete set of conservation laws of \(g^1\), we have
            \[\nabla_{\theta_1} h(\theta_1, \theta_2) \in \operatorname{span}\{\nabla_{\theta_1} h^1_i(\theta_1): i = 1, \ldots, c_1\}, \quad \forall \theta_1 \in \Theta^1.\]
    
            Similarly, for any fixed \(\theta_1\), the same argument shows that
            \[
            \nabla_{\theta_2} h(\theta_1, \theta_2) \in \operatorname{span}\{\nabla_{\theta_2} h^2_i(\theta_2): i = 1, \ldots, c_2\}, \quad \forall \theta_2 \in \Theta^2.
            \]

            Combining the two results shows that 
            \[\nabla h = (\nabla_{\theta_1} h, \nabla_{\theta_2} h) = \begin{pmatrix}
                \nabla_{\theta_1} h\\
                0
            \end{pmatrix} + \begin{pmatrix}
                0\\
                \nabla_{\theta_2} h
            \end{pmatrix}\]
            is in the span of \(\nabla \tilde h^j_i, j = 1, 2, i = 1, \ldots, c_j\).
            
            We conclude that any conservation law of the larger network \(G\) is infinitesimally generated by the extensions of the complete set of conservation laws of the subnetworks \(g^1\) and \(g^2\).
    \end{itemize}
\end{proof}

\subsection{Proof for Theorem \ref{thm:inheritance_3}}

\inheritanceIII*

\begin{proof}
    Applying Proposition \ref{prop:inheritance_1} for each subnetwork \(g^j\) shows that the extensions of the partial functional symmetries and conservation laws of \(g^j\) are also partial functional symmetries and conservation laws of \(G\).
    Applying Proposition \ref{prop:inheritance_2} recursively shows that any partial functional symmetry or conservation law of \(G\) is infinitesimally generated by the extensions of the partial functional symmetries and conservation laws of the subnetworks \(g^j\).
    
    The final step is to show that the extensions of the partial functional symmetries and conservation laws of the subnetworks \(g^j\) are independent.
    But this follows from the fact that the extensions of the partial functional symmetries and conservation laws of each subnetwork \(g^j\) are independent, and the extensions of the partial functional symmetries and conservation laws of different subnetworks \(g^j\) act on disjoint parameter spaces.

    Therefore, the extensions of the partial functional symmetries and conservation laws of the subnetworks \(g^j\) form a complete set of partial functional symmetries and conservation laws of \(G\).
\end{proof}

\section{Proof for Section \ref{sec:application}}

\subsection{Proof for Proposition \ref{prop:grouped-query-attention}}

\label{app:gqa}

To prove this result, we first recall \citep[Theorem 3.1]{tran2025equivariantneuralfunctionalnetworks}, which gives an important result on the functional symmetries of self-attention-like mechanisms.
This allows us to prove the compositional identifiability condition \ref{eqn:identifiability_3}.

\begin{theorem}[\cite{tran2025equivariantneuralfunctionalnetworks}]
    \label{thm:hoang_theorem}
    Let \(D\) be a positive integer. Assume that for a positive integer \(k\), matrices \(A_1, A_2, \ldots, A_k \in \R^{D \times D}\) and \(B_1, B_2, \ldots, B_k \in \R^{D \times D}\), we have 
    \[F(X; \{A_i, B_i\}_{i=1}^k) := \sum_{i=1}^k \operatorname{softmax}(X A_i X^\top) X B_i = 0,\]
    for all positive integers \(L\) and \(X \in \R^{L \times D}\). Then, if \(A_1, A_2, \ldots, A_k\) are pairwise distince, then
    \[B_1 = \cdots = B_k = 0.\]
\end{theorem}

We also summarize \cite[Section 4.1]{marcotte2024abidelawfollowflow}, which gives a complete characterization of indepdendent conservation laws of the matrix factorization \(U V^\top\).

\begin{lemma}
    \label{lem:UV_lemma}
    Let \(U \in \R^{n \times r}, V \in \R^{m \times r}\) be full column rank matrices. 
    Then, the following upper triangular components of the following matrix function form a complete set of independent conservation laws of the matrix factorization \(\phi(U,V) = U V^\top\):
    \[H(U, V) = U^\top U - V^\top V.\]
\end{lemma}
\begin{proof}
    The definition of the matrix factorization \(U V^\top\) is given at the beginning of Section 4.1 in \cite{marcotte2024abidelawfollowflow}.
    The final corollary of \cite[Section 4.1]{marcotte2024abidelawfollowflow} shows that the functions \(H(U, V)\) form a complete set of independent conservation laws of the matrix factorization \(U V^\top\).

    Intuitively, the results here can be interpretted as follows.
    The matrix factorization \(U V^\top\) is invariant under the group action of \(\operatorname{GL}_{r}(\R)\) given by 
    \[S \circ (U, V) = (US, V(S^\top)^{-1}).\]
    The infinitesimal generators of this Lie group generate a codistribution \(W^\perp = \operatorname{ker} D \phi\). 
    Its orthogonal complement \(W = \operatorname{range}(D \phi^\top)\) is the 'gradient flow distribution' of the matrix factorization \(U V^\top\).
    By calculating the rank of \(\operatorname{Lie}(W)\) under the assumptions of the lemma, the authors showed that the rank of \(\operatorname{Lie(W)}^\perp\) matches the number of independent functions defined by \(H(U, V)\), and hence they form a complete set of independent conservation laws.
    This result does not rely on the specific form of the loss function, as it is a property of the matrix factorization \(U V^\top\) itself.
\end{proof}

We can now proceed with the proof of Proposition \ref{prop:grouped-query-attention}.

\groupedQueryAttention*

\begin{proof}
    Let \(\Theta\) be an open neighborhood of \(\theta_0\) in the parameter space of \(G\) such that the assumptions hold for all \(\theta \in \Theta\).
    Choose small open balls around the distinct attention matrices \(Q^j_i (K^j)^\top\) such that they are pairwise disjoint.
    Shrink \(\Theta\) if necessary so that for all \(\theta \in \Theta\), the attention matrices \(Q^j_i (K^j)^\top\) are contained in the corresponding disjoint balls.
    This ensures that functional symmetries induced by permutations of the heads are not allowed, and hence the attention matrices \(Q^j_i (K^j)^\top\) are pairwise distinct for all \(\theta \in \Theta\).
    This will be the local neighborhood in which we prove our results.

    \textbf{Step 1}. Identifiability result.
    
    The main tool is \cite[Theorem 3.1]{tran2025equivariantneuralfunctionalnetworks}.
    For other networks, one would need to prove a similar result for the specific architecture.

    Let 
    \[\overline \theta = (\overline Q^j_1, \ldots, \overline Q^j_{k}, \overline K^j, \overline V^j, \overline O^j_1, \ldots, \overline O^j_k)_{j=1}^{n_G} \in \Theta\]
    be another set of parameters such that \(G(\theta, \cdot) = G(\overline \theta, \cdot)\).
    Then, for all \(X\) with arbitrary number of tokens, each of size \(D\), we have
    \[\sum_{j=1}^{n_G} \sum_{i=1}^{k} \operatorname{softmax}\left(X\frac{Q^j_i (K^j)^\top }{\sqrt{D_h}}X^\top\right) X V^j (O^j_i)^\top 
    - \sum_{j=1}^{n_G} \sum_{i=1}^{k} \operatorname{softmax}\left(X\frac{ \overline Q^j_i (\overline K^j)^\top }{\sqrt{D_h}}X^\top\right)X \overline V^j (\overline O^j_i)^\top 
    = 0.\]
    Since the term on the left contains \(n_H = k n_G\) terms with pairwise distinct attention matrices \(Q^j_i (K^j)^\top\) and nonzero weight matrices \(V^j (O^j_i)^\top\), Theorem \ref{thm:hoang_theorem} implies that the attention matrices and weight matrices must be equal, i.e.,
    \[
        Q^j_i (K^j)^\top = \overline Q^j_i (\overline K^j)^\top, \quad 
        V^j (O^j_i)^\top = \overline V^j (\overline O^j_i)^\top.
    \]
    By the matrix factorization result \cite[Theorem 2]{piziak1999fullrankfactorization} and the full rank assumption, we have that for each \(j = 1, \ldots, n_G\), \(i = 1, \ldots, k\) there exists invertible matrices \(S^j_i \in \R^{D_h \times D_h}, T^{j}_i \in \R^{D_h \times D_h}\) such that
    \[
        \overline Q^j_i = Q^j_i S^j_i, \quad 
        \overline K^j = K^j ((S^j_i)^\top)^{-1}, \quad
        \overline V^j = V^j T^j_i, \quad 
        \overline O^j_i = O^j_i ((T^j_i)^\top)^{-1}.
    \]
    But since \(K^j, V^j\) are full rank, we must have 
    \[
    S^j_1 = \cdots = S^j_k =: S^j, \quad T^j_1 = \cdots = T^j_k =: T^j
    \] 
    for each \(j = 1, \ldots, n_G\). 

    It is easy to verify that for any invertible matrices \(S^j, T^j \in \R^{D_h \times D_h}\), the following transformations of the parameters leave the grouped query attention mechanism \(G\) functionally invariant:
    \[\overline Q^j_i = Q^j_i S^j, \quad \overline K^j = K^j ((S^j)^\top)^{-1}, \quad \overline V^j = V^j T^j, \quad \overline O^j_i = O^j_i ((T^j)^\top)^{-1}.\]

    Thus, all functionally equivalent parameters of the grouped query attention mechanism \(G\) are given by the following group action of \((\operatorname{GL}_{D_h}(\R) \times \operatorname{GL}_{D_h}(\R))^{n_G}\):
    \[
        (S^j, T^j)_{j=1}^{n_G} \cdot (Q^j_1, \ldots, Q^j_{k}, K^j, V^j, O^j_1, \ldots, O^j_k)_{j=1}^{n_G} 
        = (\tilde Q^j_1, \ldots, \tilde Q^j_{k}, \tilde K^j, \tilde V^j, \tilde O^j_1, \ldots, \tilde O^j_k)_{j=1}^{n_G},
    \]
    where
    \[
        \tilde Q^j_i = Q^j_i S^j, \quad 
        \tilde K^j = K^j ((S^j)^\top)^{-1}, \quad
        \tilde V^j = V^j T^j, \quad 
        \tilde O^j_i = O^j_i ((T^j)^\top)^{-1}.
    \]

    \textbf{Step 2}. Reduction to subnetworks

    The main tool is Theorem \ref{thm:inheritance_3} from our framework. 
    It is reusable for other architectures, as long as the compositional identifiability condition \ref{eqn:identifiability_3} holds.

    Stack the query matrices \(Q^j_i\) and output matrices \(O^j_i\) for each group \(j\) into the matrices
    \[
        Q^j = \begin{pmatrix}
            Q^j_1\\
            \vdots\\
            Q^j_k
        \end{pmatrix} \in \R^{k D \times D_h}, \quad 
        O^j = \begin{pmatrix}
            O^j_1\\
            \vdots\\
            O^j_k
        \end{pmatrix} \in \R^{k D \times D_h}.
    \]
    Define the following functions
    \begin{align*}
        g^j_1(Q^j_1, \ldots, Q^j_{k}, K^j) = g^j_1(Q^j, K^j) 
        &:= Q^j (K^j)^\top = \begin{pmatrix}
            Q^j_1 (K^j)^\top\\
            \vdots\\
            Q^j_k (K^j)^\top
        \end{pmatrix} \in \R^{kD \times D}, \\
        g^j_2(V^j, O^j_1, \ldots, O^j_{k}) = g^j_2(V^j, O^j)
        &= V^j (O^j)^\top = (V^j (O^j_1)^\top, \ldots, V^j (O^j_{k})^\top) \in \R^{D \times kD}.
    \end{align*}
    By the identifiability result in Step 1, the compositional identifiability conditon \ref{eqn:identifiability_3} holds for the grouped query attention mechanism \(G\) with respect to the functions \(g^j_1\) and \(g^j_2\) for all \(j = 1, \ldots, n_G\).
    Thus, by Theorem \ref{thm:inheritance_3}, the complete set of symmetries and conservation laws of \(G\) can be obtained by combining the extensions of complete sets of symmetries and conservation laws of \(g^j_1\) and \(g^j_2\) for all \(j = 1, \ldots, n_G\).

    \textbf{Step 3}. Complete characterization of conservation laws of subnetworks.
    
    The main tool is Lemma~\ref{lem:UV_lemma}. 
    For other networks, one would need to prove a similar result for the specific architecture.

    By Lemma \ref{lem:UV_lemma}, the complete set of independent conservation laws of \(g^j_1\) and \(g^j_2\) are given by
    \begin{align*}
        H^j_1(Q^j, K^j) &= (Q^j)^\top Q^j - (K^j)^\top K^j = \sum_{i=1}^{k} (Q_i^j)^\top Q^j_i - (K^j)^\top K^j \\
        H^j_2(V^j, O^j) &= (V^j)^\top V^j - (O^j)^\top O^j = (V^j)^\top V^j - \sum_{i=1}^k (O^j_i)^\top O^j_i.        
    \end{align*}
    Applying Theorem \ref{thm:inheritance_3}, we obtain the complete set of independent conservation laws of \(G\), as stated in the proposition.
\end{proof}

\subsection{Proof for Proposition \ref{prop:PNN}}
\label{sec:proof:PNN}

\PNN*

\begin{proof}
    We recall from \cite{usevich2025identifiability} that a polynomial neural network with a specific parameter \(\theta_0\) is called finite-to-one if there are only finitely many functionally equivalent classes of parameters under the following transformations
    \[W'_l = P_l D_l W_l D_{l-1}^{-r_{l-1}} P_{l-1}^\top, \quad b'_l = P_l D_l b_l,\]
    where \(P_l \in \mathbb Z^{D_l \times D_l}\) is a permutation matrix, \(D_l \in \mathbb R^{D_l \times D_l}\) is a diagonal matrix, and \(P_0 = D_0 = I, P_L = D_L = I\).

    We first restrict our attention to a small enough neighborhood of \(\theta_0\) to eliminate all discrete symmetries of the PNN mechanism \(G\).
    Let \(\Theta\) be an open neighborhood of \(\theta_0\) in the parameter space of \(G\) such that it avoids all other functionally equivalent classes of parameters.
    Further shrink \(\Theta\) if necessary so permutations by \(P_l, P_{l-1}\) are not allowed.
    Then, by the finite-to-one assumption, any functionally equivalent parameters \(\theta' \in \Theta\) must be related to \(\theta\) by a transformation of the form
    \[W'_l = D_l W_l D_{l-1}^{-r_{l-1}}, \quad b'_l = D_l b_l,\]
    where \(D_l \in \mathbb R^{D_l \times D_l}\) is an invertible diagonal matrix, and \(D_0 = D_L = I\).
    This will be the local neighborhood in which we prove our results.

    In other words, the functionally equivalent parameters are given by the partial group action of \((\mathbb R^*)^{D_1} \times \cdots \times (\mathbb R^*)^{D_{L-1}}\) on the parameter space of \(G\) given by
    \begin{align*}
        &(t_{l, i})^{l = \overline{1, L-1}}_{i= \overline{1, D_l}} \circ (W_l, b_l)_{l=1}^L \\  
        = &\left((t_{l, 1} E_{1} +\cdots + t_{l, D_l} E_{D_l}) W_l (t_{l-1, 1}^{-r_{l-1}} E_{1} + \cdots + t_{l-1, D_{l-1}}^{-r_{l-1}} E_{D_{l-1}}), (t_{l, 1} E_{1} +\cdots + t_{l, D_l} E_{D_l}) b_l\right),
    \end{align*}
    where \(\mathbb R^* = \mathbb R \setminus \{0\}\) is the multiplicative group of nonzero real numbers and \(E_{i}\) is the square matrix of the appropriate size with \(1\) in the \((i, i)\)-th entry and \(0\) elsewhere.
    This Lie group has dimension \(s:=\sum_{j=1}^{L-1} D_j\), so the number of partial independent functional symmetries is at most \(s\).
    By Assumption \ref{as:separability} and Proposition \ref{prop:separability-symmetry}, the number of partial independent loss symetries is also at most \(s\).
    Finally, by Theorem \ref{thm:law-symmetry-gap}, the number of independent conservation laws is at most \(s\).

    We will show that there exists exactly this many independent conservation laws, which then must form a complete set of conservation laws of the PNN mechanism \(G\).    
    Observe that the following functions are functional symmetries of the PNN mechanism \(G\):
    \[
    \psi^j_i(t, \theta) = (W_1, b_1, \ldots, e^{t E_{i}} W_{j}, e^{t E_i} b_j, W_{j+1} e^{-t r_{l} E_{i}}, \ldots, W_L, b_L), \; j = \overline{1, L-1},\; i = \overline{1, D_j}.
    \]
    Their infinitesimal generators are given by
    \[
    \chi^j_i(\theta) = D_t \psi^j_i(0, \theta) = (0, 0, \ldots, E_{i} W_{j}, E_i b_j, -r_{l} W_{j+1} E_{i}, \ldots, 0, 0), \; j = \overline{1, L-1},\; i = \overline{1, D_j}.
    \]
    By the assumption that \(\theta_0\) is generic and the fact that each \(\chi^j_i\) involves only the \(i\)-th row of \(W_j\) and the \(i\)-th column of \(W_{j+1}\), these vector fields must be linearly independent.
    It can be verified that the functions 
    \[h^j_i(\theta) = \|(W_j)_{i,:}\|_2^2 + (b_j)_{i}^2 - r_j \|(W_{j+1})_{:,i}\|_2^2, \quad i = \overline{1, D_j}, j = \overline{1, L-1}\]
    have gradients \(\nabla h^j_i = 2 \chi^j_i\).
    By Proposition \ref{prop:symmetry_vector_field} and Proposition \ref{prop:conservation_law_vector_field}, the functions \(h^j_i\) are conservation laws of the PNN mechanism \(G\).
    Since their gradients \(\nabla h^j_i\) are linearly independent, the functions \(h^j_i\) are also independent.
    Finally, since there are exactly \(\sum_{j=1}^{L-1} D_j\) independent functions \(h^j_i\), they form a complete set of conservation laws of the PNN mechanism \(G\).
\end{proof}

\subsection{Proof for Proposition \ref{prop:deep_linear_network}}

\deepLinearNetwork*

\begin{proof}
    Lemma 1 from \cite{lindsey2026regularizationimpliesbalancednessdeep} shows that all symmetries of the deep square linear network is given by the group action of \((\operatorname{GL}_D(\R))^{L-1}\) on the parameters \((W_1, \ldots, W_L)\) via
    \[(S_1, \ldots, S_{L-1}) \cdot (W_1, \ldots, W_L) = (S_1 W_1, S_2 W_2 S_1^{-1}, \ldots, S_{L-1} W_{L-1} S_{L-2}^{-1}, W_L S_{L-1}^{-1}).\]
    Since the dimension of \((\operatorname{GL}_D(\R))^{L-1}\) is \((L-1) D^2\), Theorem \ref{thm:law-symmetry-gap} implies that there are at most \((L-1) D^2\) independent conservation laws of \(G\) with respect to any loss satisfying Assumption~\ref{as:separability}.
\end{proof}

\section{Learning to Scale}
\label{app:learning_to_scale}

Consider the problem of using a two layer linear network \(G\) to learn a scaling function \(f\) given by
 \[G((U, V), x) = UV^\top x, \quad f(x) = \alpha x = \alpha I x.\]
We specifically focus on the cases where \(U, V \in \R\) or \(U, V \in \R^{2 \times 2}\), and the loss function is given by the mean squared error
\[\ell((U, V), x) = \|G((U, V), x) - f(x)\|_2^2 = \|UV^\top x - \alpha x\|_2^2.\]

\subsection{The scalar case}

In the scalar case, a complete set of independent conservation laws of the two layer linear network \(G\) is given by a single scalar function
\[h(U, V) = U^2 - V^2.\]
The symmetry-law gap is \(0\). 
This informs us that this conservation law should be sufficient to fully determines the local minimizer that gradient flow will converge to given any initial condition.

Indeed, assume that \(\alpha > 0\), and that the weights are initialized at \(U_0 > 0, V_0 > 0\).
Then the conservation law value is given by
\[h(U_0, V_0) = U_0^2 - V_0^2 = h_0 \in \R.\]
If gradient flow converges, it will converge to a minimizer \((U^\ast, V^\ast)\) with the same conservation law value.
This gives us the following system of equations:
\[\begin{cases}
    (U^\ast)^2 - (V^\ast)^2 = h_0 \\ 
    U^\ast V^\ast = \alpha,
\end{cases}\]
which yields a degree 4 polynomial equation in \(U^\ast\) via the substitution \(V^\ast = \alpha / U^\ast\). 
Solving this gives a unique solution in the positive quadrant (the same quadrant as the initial condition) given by
\begin{equation}
    \label{eqn:scalar_exact}
    (U^\ast, V^\ast) = \left(\sqrt{\frac{h_0 + \sqrt{h_0^2 + 4 \alpha^2}}{2}}, \frac{\alpha}{\sqrt{\frac{h_0 + \sqrt{h_0^2 + 4 \alpha^2}}{2}}}\right).
\end{equation}
Thus, we can predict that if gradient flow converges, it will converge to the unique minimizer \((U^\ast, V^\ast)\) given by \eqref{eqn:scalar_exact}, which is fully determined by the conservation law value \(h_0\) at initialization.

\subsection{The 2x2 matrix case}

In the \(2 \times 2\) matrix case, a complete set of independent conservation laws of the two layer linear network \(G\) is given by the following matrix function:  
\[
H(U, V) = U^\top U - V^\top V = \begin{pmatrix}
    h_1(U,V) & h_2(U,V)\\ 
    h_2(U,V) & h_3(U,V)
\end{pmatrix} . 
\]
The symmetry-law gap is \(1\). 
This informs us that the conservation laws are not sufficient to fully determine the local minimizer that gradient flow will converge to given any initial condition.
Rather, gradient flow will converge to a point on a one dimensional manifold of local minimizers that share the same conservation law values.

Indeed, assume that the weights are initialized at \(U_0, V_0 \in \R^{2 \times 2}\), and that gradient flow converges to a minimizer \((U^\ast, V^\ast)\) with the same conservation law values.
This gives us the following system of equations:
\begin{align}
    \label{eqn:2by2law} H(U, V) &= H(U_0, V_0) := H_0\\ 
    \label{eqn:2by2optimal} U V^\top &= \alpha I.
\end{align}
Solving this system of equations yield a one dimensional manifold of solutions induced by the orthogonal group \(\operatorname{O}(2)\).
We present this in Lemma \ref{lem:2by2exactsol}, whose proof we defer to the end of this section. 

\begin{restatable}{lemma}{exactsol}
    \label{lem:2by2exactsol}
    Let \(H_0 \in \R^{2 \times 2}\) be a symmetric matrix and \(0 \neq \alpha \in \R\) be given.
    The entire solution set of \((U, V) \in (\R^{2 \times 2} \times \R^{2 \times 2})\) that satisfies \eqref{eqn:2by2law} and \eqref{eqn:2by2optimal} is given by the group action of the orthogonal group \(\operatorname{O}(2)\) on the solution \((P^{1/2}, \alpha (P^{1/2})^{-\top})\):
    \[
    Q \cdot (P^{1/2}, \alpha P^{-1/2}) = (Q P^{1/2}, \alpha Q P^{-1/2}), \quad Q \in \operatorname{O}(2), 
    \]
    where \(P = \frac{H_0 + \sqrt{H_0^2 + 4 \alpha^2 I}}{2}\) is a positive definite matrix.
\end{restatable}

Thus, the conservation laws are not sufficient to fully determine the local minimizer that gradient flow will converge to given any initial condition. 
Nevertheless, they are still sufficient to determine important numerical properties of the minimizer, such as the singular values or condition numbers of \(U^\ast\) and \(V^\ast\). 
In particular, since these properties are invariant under the action of the orthogonal group \(\operatorname{O}(2)\), the singular values and condition numbers of \(U^\ast\) (resp. \(V^\ast\)) are the same as those of \(P^{1/2}\) (resp. \(\alpha P^{-1/2}\)), where 
\begin{equation}
    \label{eqn:P_2by2}
    P = \frac{H_0 + \sqrt{H_0^2 + 4 \alpha^2 I}}{2}
\end{equation}
is determined at initialization by the conservation law value \(H_0\).
Therefore, if one desires to have a well-conditioned minimizer \((U^\ast, V^\ast)\), one should initialize the weights such that \(P\) is well-conditioned, e.g., \(P \approx \lambda I\) for some \(\lambda > 0\).

\begin{proof}[Proof of Lemma~\ref{lem:2by2exactsol}]
    To solve this system, note that the condition \ref{eqn:2by2optimal} implies \(V = \alpha U^{-\top}\), and so \eqref{eqn:2by2law} is equivalent to
    \[H_0 = U^\top U -  V^\top V=   U^\top U - \alpha^2 U^{-1} U^{-\top} .\]
    Setting the positive definite Gram matrix \(P :=  U^\top U\), we obtain the equivalent equation
    \begin{equation}
        \label{eqn:2by2lawequiv}
        H_0 = P - \alpha^2 P^{-1}.
    \end{equation}
    Let \(R\) be a orthogonal matrix that diagonalizes \(P\), i.e., \(R^\top P R = D\) where \(D\) is a diagonal matrix.
    We can show that \(R\) also diagonalizes \(H_0\) as follows:
    \[R^\top H_0 R = R^\top (P - \alpha^2 P^{-1}) R = D - \alpha^2 D^{-1}.\]
    Thus, to solve for \(P\) in \eqref{eqn:2by2lawequiv} is equivalent to solving for \(D\) in
    \[D - \alpha^2 D^{-1} = R^\top H_0 R,\]
    but this reduces to solving a quadratic equation for each diagonal entry of \(D\), which has a unique positive diagonal matrix solution
    \[D = \frac{R^\top H_0 R + \sqrt{(R^\top H_0 R)^2 + 4 \alpha^2 I}}{2}.\]
    By reversing the diagonalization, we obtain that
    \[P = R D R^\top = \frac{H_0 + \sqrt{H_0^2 + 4 \alpha^2 I}}{2}.\]
    Finally, since \(P\) is the Gram matrix of \(U\), it is known that \(U\) must be of the form \(U= Q P^{1/2}\) for some orthogonal matrix \(Q\).
    In other words, the solution set of \((U, V)\) that satisfies \eqref{eqn:2by2law} and \eqref{eqn:2by2optimal} must be given by the group action of the orthogonal group \(\operatorname{O}(2)\) on the solution \((P^{1/2}, \alpha (P^{1/2})^{-\top})\):
    \[Q \cdot (P^{1/2}, \alpha P^{-1/2}) = (Q P^{1/2}, \alpha Q P^{-1/2}).\]
    It is easy to verify that these are indeed solutions, and so they are all the solutions.
\end{proof}

\section{Empirical Examination of Conservation Laws in a Language Model}
\label{app:exp2}

In this section, we empirically examine the conservation laws of the multi-head self-attention and grouped query attention mechanisms during the training of a language model with SGD.
The theoretical basis for this experiment is given by Proposition 5.1 of \cite{marcotte2025transformativeconservativeconservationlaws}, which states that for constant step size, the conservation law error grows linearly with the number of step size.

We use a modified version of the character level nanoGPT model \citep{karpathy2022nanogpt} with 6 layers, each with 6 attention heads.
Layers 0, 2, 3, 5 are standard multi-head self-attention layers, while layers 1 and 4 are grouped query attention layers with 3 groups of 2 query heads.
We train from scratch on the Tiny Shakespeare dataset \citep{karpathy2015tinyshakespeare} for 20,000 updates using plain SGD with a constant learning rate of \(10^{-2}\) and no weight decay.
We set dropout to 0 and use training without gradient clipping to avoid any potential interference with the conservation laws. Training was performed on a single NVIDIA RTX 3080 GPU with 10GB of memory, and took less than 2 hours to complete.

\begin{figure}
    \centering
    \includegraphics[width= 0.9\textwidth]{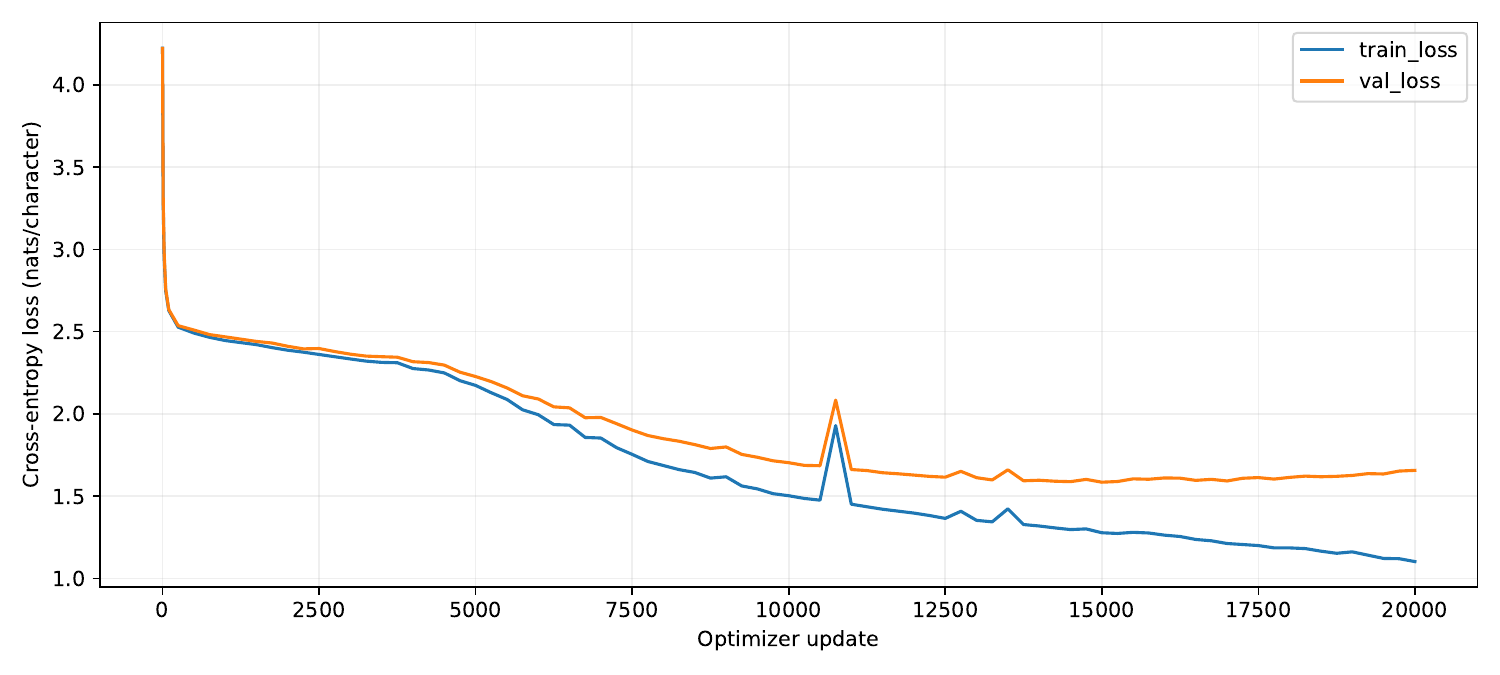}
    \caption{Training and validation loss of the language model with SGD. Both losses decrease substantially over the course of training.}
    \label{fig:llm_loss}
\end{figure}

The training loss and validation loss are shown in Figure \ref{fig:llm_loss}, which shows that the model is learning effectively.
Sample outputs from the language model with the prompt "KING:" before training and after 20,000 updates are shown in Figure \ref{fig:language-samples}.

\lstdefinestyle{lmsample}{
  basicstyle=\ttfamily\scriptsize,
  breaklines=true,
  columns=fullflexible,
  keepspaces=true,
  showstringspaces=false,
  frame=single,
  framerule=0.3pt
}

\begin{figure}[t]
\centering

\begin{minipage}[t]{0.48\linewidth}
\textbf{\texttt{KING:} step 0}
\begin{lstlisting}[style=lmsample]
KING:iW,xV!fF,vOTTtcukUaGBS&HZ?
JGP!OOT WuurJ'uuWX-o--pGM3???J?Hh-a;aHwu&z:upupIIyy;-,NllkwmGFTOggmkHjg$yiPYqeMlbCLSYOeUM;JzlUIqhETTHH:c:g;Vre&dN-pugJo?aHeMhnfSfHH ks&jn''IIM''d$uS'S
d:qsgJB-Q:o
SSJB$,ssTI;mLs?VVoN:VN;zsm.HqgcPVWhzn
?Ghjs?u;Yje

\end{lstlisting}
\end{minipage}\hfill
\begin{minipage}[t]{0.48\linewidth}
\textbf{\texttt{KING:} step 20,000}
\begin{lstlisting}[style=lmsample]
KING:
What from Tutualte the oratops of this business be.

GLOUCESTER:
The quarrel of thy smooth, many by thee,
So hated, and I have shed unrest, that have heard
A knotny before the brother than words,
When he has has want to woe Romans soul of
\end{lstlisting}
\end{minipage}

\caption{Opening excerpts from fixed-prompt samples before training and
after 20,000 optimizer updates. The language model has learned to produce text that contains recognizable words and dialogue-like formatting, demonstrating that it has learned some structure of the English language.}
\label{fig:language-samples}
\end{figure}

For each layer and each group \(j\), define the Gram matrices of the query, key, value, and output weights as
$$G_Q^j=\sum_{i\in j}Q_i^\top Q_i,\quad G_K^j=K_j^\top K_j,
\qquad G_V^j=V_j^\top V_j,\quad G_O^j=\sum_{i\in j}O_i^\top O_i.$$
Following the setting in Appendix E of \cite{tran2026conservation}, we track both the conservation law matrices and a non-conserved sum of Gram matrices as a control:
$$H_{QK}^j=G_Q^j-G_K^j,\quad H_{VO}^j=G_V^j-G_O^j,
\qquad C_{QK}^j=G_Q^j+G_K^j,\quad C_{VO}^j=G_V^j+G_O^j.$$
While the conservation law matrices \(H_{QK}^j\) and \(H_{VO}^j\) are expected to be approximately preserved during training, the non-conserved sum matrices \(C_{QK}^j\) and \(C_{VO}^j\) are not expected to be preserved.

\begin{figure}[ht!]
    \centering
    \includegraphics[width=\textwidth]{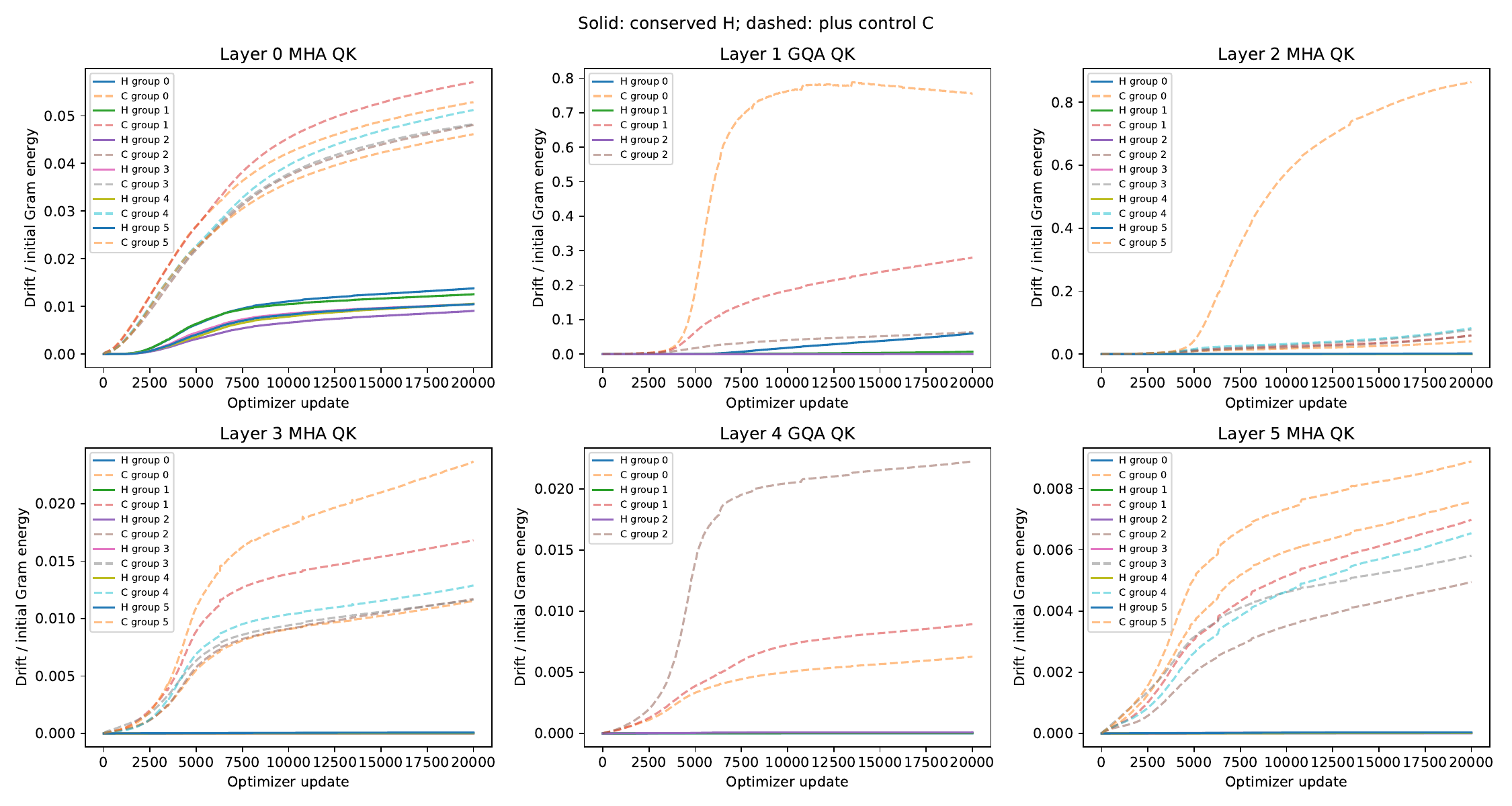}
    \caption{Conservation law error and non-conserved sum error for the \(Q, K\) blocks of the multi-head self-attention and grouped query attention mechanisms during training of a language model with SGD. Solid curves show the normalized drift of conservation laws, while dashed curves show the normalized drift of the non-conserved sum. The results show that the conservation law error drifts much less than the non-conserved sum error.}
    \label{fig:llm_conservation_law}
\end{figure}

\begin{figure}[ht!]
    \centering
    \includegraphics[width=\textwidth]{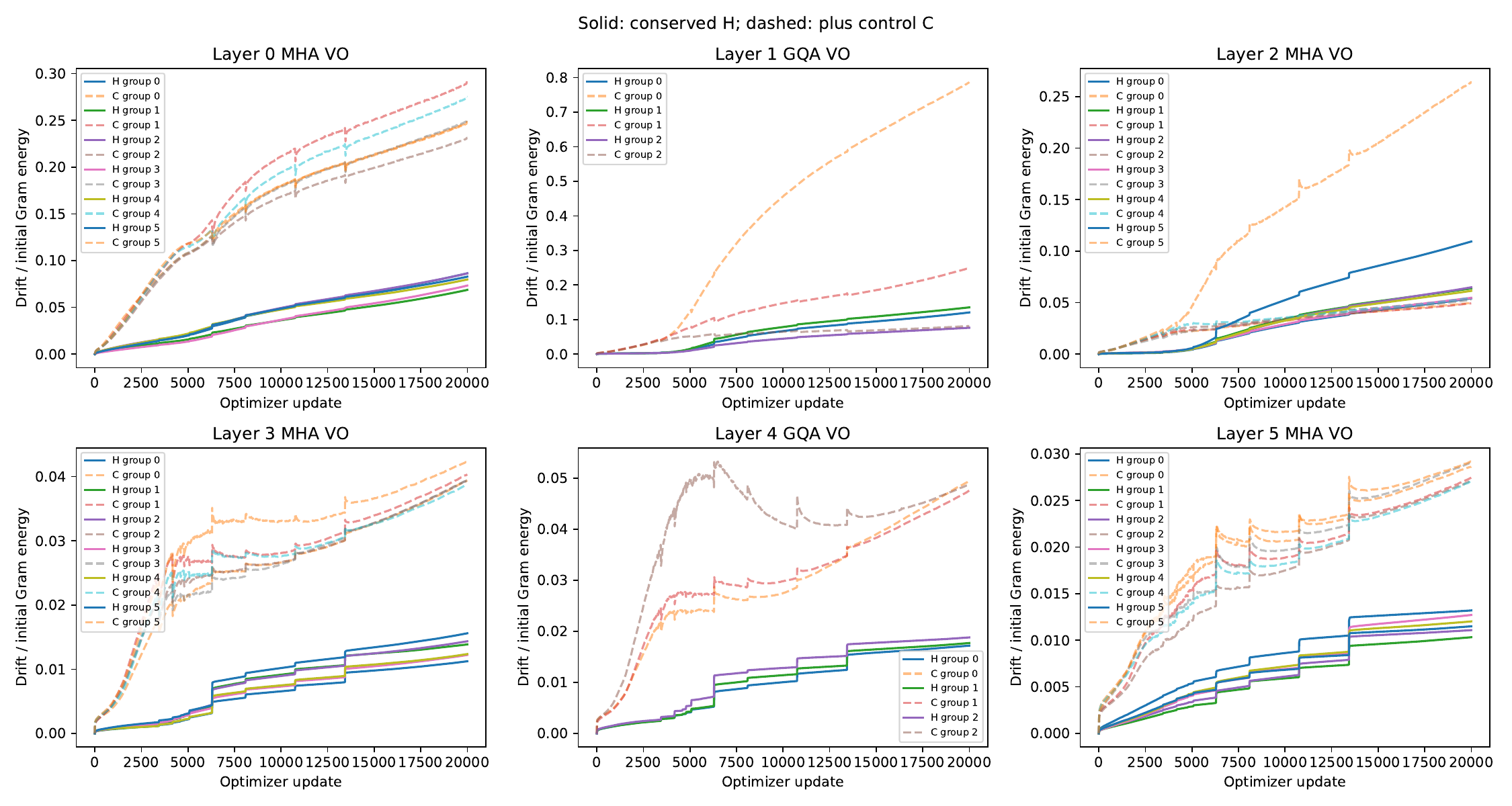}
    \caption{Conservation law error and non-conserved sum error for the \(V, O\) blocks of the multi-head self-attention and grouped query attention mechanisms during training of a language model with SGD. Solid curves show the normalized drift of conservation laws, while dashed curves show the normalized drift of the non-conserved sum. For most layers, the result shows that the conservation law error drift much less than the non-conserved sum error. However, for the \(V, O\) blocks in layers 1 and 2, this trend only holds for the first 5,000 updates.}
    \label{fig:llm_conservation_law_2}
\end{figure}

To compare the conservation law error and non-conserved sum error across different layers and groups, we normalize the errors by the initial scale of the Gram matrices
$$s_{QK}^j=\|G_Q^j(0)\|_F+\|G_K^j(0)\|_F,
\qquad s_{VO}^j=\|G_V^j(0)\|_F+\|G_O^j(0)\|_F.$$
and plot the normalized conservation law error
\[\frac{\|H_{QK}^j(t) - H_{QK}^j(0)\|_F}{s_{QK}^j},\quad \frac{\|H_{VO}^j(t) - H_{VO}^j(0)\|_F}{s_{VO}^j}\]
as well as the normalized non-conserved sum error
\[\frac{\|C_{QK}^j(t) - C_{QK}^j(0)\|_F}{s_{QK}^j},\quad \frac{\|C_{VO}^j(t) - C_{VO}^j(0)\|_F}{s_{VO}^j}.\]
The implementation adds \(10^{-12}\) to each denominator to avoid division by zero.

We report the results in Figure \ref{fig:llm_conservation_law} for the \(Q, K\) blocks and Figure \ref{fig:llm_conservation_law_2} for the \(V, O\) blocks.
We observe that many conservation law matrices exhibit a smaller drift than the non-conserved sum matrices, particularly for the \(Q, K\) blocks.
However, for the \(V, O\) blocks in layers 1 and 2, this trend only holds for the first 5,000 updates.

This experiment illustrates that the conservation laws can be approximately preserved during training with SGD, but the degree of preservation can vary across different layers and blocks of the model, and may deteriorate over time. 

\end{document}